\documentclass[twoside]{article}

\usepackage[accepted]{aistats2019}
\usepackage[numbers]{natbib}
\usepackage{amsthm,amsmath,amsfonts,natbib,graphicx,bm,booktabs, soul, cancel, mathtools}
\usepackage{mathabx}

\numberwithin{equation}{section}
\theoremstyle{plain}
\newtheorem{thm}{Theorem}[section]
\newtheorem{lem}{Lemma}[section]
\newtheorem{prop}{Proposition}[section]
\newtheorem{cor}{Corollary}[section]
\newtheorem{remark}{Remark}[section]
\newtheorem{defn}{Definition}

\DeclareMathOperator*{\argmin}{arg\,min}

\DeclareMathOperator*{\E}{\mathbb{E}}

\DeclareMathOperator{\R}{\mathcal{R}}
\def\F{\mathcal{F}}
\newcommand{\vertiii}[1]{{\vert\kern-0.1ex\vert\kern-0.1ex\vert #1 
    \vert\kern-0.1ex\vert\kern-0.1ex\vert}}
\newcommand{\RNum}[1]{\uppercase\expandafter{\romannumeral #1\relax}}

\def\Hess{\mathop{\rm Hess}}

\begin{document}



\twocolumn[

\aistatstitle{Fisher-Rao Metric, Geometry, and Complexity of Neural Networks}

\aistatsauthor{ Tengyuan Liang \And Tomaso Poggio \And  Alexander Rakhlin \And James Stokes }

\aistatsaddress{ University of Chicago \And  MIT \And MIT \And University of Pennsylvania }
]

\begin{abstract}
We study the relationship between geometry and capacity measures for deep neural networks from an invariance viewpoint. We introduce a new notion of capacity\,---\,the Fisher-Rao norm\,---\,that possesses desirable invariance properties and is motivated by Information Geometry. We discover an analytical characterization of the new capacity measure, through which we establish norm-comparison inequalities and further show that the new measure serves as an umbrella for several existing norm-based complexity measures. We discuss upper bounds on the generalization error induced by the proposed measure. Extensive numerical experiments on CIFAR-10 support our theoretical findings. Our theoretical analysis rests on a key structural lemma about partial derivatives of multi-layer rectifier networks.
\end{abstract}

\section{Introduction}
Beyond their remarkable representation and memorization ability, deep neural networks empirically perform well in out-of-sample prediction. This intriguing out-of-sample generalization property poses two fundamental theoretical questions: (1) What are the complexity notions that control the generalization aspects of neural networks? (2) Why does stochastic gradient descent, or other variants, find  parameters with small complexity?

In this paper we approach the generalization question for deep neural networks from a geometric invariance vantage point.  The motivation behind invariance is twofold: (1) The specific parametrization of the neural network is arbitrary and should not impact its generalization power. As pointed out in \citep{neyshabur2015path}, for example, there are many continuous operations on the parameters of ReLU nets that will result in exactly the same prediction and thus generalization can only depend on the equivalence class obtained by identifying parameters under these transformations. (2) Although flatness of the loss function has been linked to generalization \citep{hochreiter1997flat}, existing definitions of flatness are neither invariant to nodewise re-scalings of ReLU nets nor general coordinate transformations \citep{dinh2017sharp} of the parameter space, which calls into question their utility for describing generalization.

It is thus natural to argue for a purely geometric characterization of generalization that is invariant under the aforementioned transformations and additionally resolves the conflict between flat minima and the requirement of invariance. 
%More generally, we are in pursuit of a certain geometry in the parameter space that is ``invariant'' to re-parametrization of the problem. 
Information geometry is concerned with the study of geometric invariances arising in the space of probability distributions, so we will leverage it to motivate a particular geometric notion of complexity\,---\,the Fisher-Rao norm. From an algorithmic point of view the steepest descent induced by this geometry is precisely the natural gradient \citep{amari1998natural}.  From the generalization viewpoint, the Fisher-Rao norm naturally incorporates distributional aspects of the data and harmoniously unites elements of flatness and norm which have been argued to be crucial for explaining generalization \citep{neyshabur2017exploring}. 

Statistical learning theory equips us with many tools to analyze out-of-sample performance. The Vapnik-Chervonenkis dimension is one possible complexity notion, yet it may be too large to explain generalization in over-parametrized models, since it scales with the size (dimension) of the network. 
In contrast, under additional distributional assumptions of a margin, Perceptron (a one-layer network) enjoys a dimension-free error guarantee, with the $\ell_2$ norm playing the role of ``capacity''. These observations (going back to the 60's) have led to the theory of large-margin classifiers, applied to kernel methods, boosting, and neural networks \citep{anthony1999neural}. 
In particular, the analysis of \cite{koltchinskii2002empirical}  combines the empirical margin distribution (quantifying how well the data can be separated) and the Rademacher complexity of a restricted subset of functions. This in turn raises the capacity control question: what is a good notion of the restrictive subset of parameter space for neural networks? Norm-based capacity control provides a possible answer and is being actively studied for deep networks \citep{krogh1992simple, neyshabur2015norm, neyshabur2015path, bartlett2017spectrally, neyshabur2017exploring}, yet the invariances are not always reflected in these capacity notions. In general, it is very difficult to answer the question of which capacity measure is superior. Nevertheless, we will show that our proposed Fisher-Rao norm serves as an umbrella for the previously considered norm-based capacity measures, and it appears to shed light on possible answers to the above question. 

Much of the difficulty in analyzing neural networks stems from their unwieldy recursive definition interleaved with nonlinear maps. In analyzing the Fisher-Rao norm, we proved an identity for the partial derivatives of the neural network that appears to open the door to some of the geometric analysis. In particular, we prove that any stationary point of the empirical objective with hinge loss that perfectly separates the data must also have a large margin. Such an automatic large-margin property of stationary points may link the algorithmic facet of the problem with the generalization property. The same identity gives us a handle on the Fisher-Rao norm and allows us to prove a number of facts about it. Since we expect that the identity may be useful in deep network analysis, we start  by stating this result and its implications in the next section. In Section \ref{sec:fr} we introduce the Fisher-Rao norm and establish through norm-comparison inequalities that it serves as an umbrella for existing norm-based measures of capacity. Using these norm-comparison inequalities we bound the generalization error of various geometrically distinct subsets of the Fisher-Rao ball and provide a rigorous proof of generalization for deep linear networks. Extensive numerical experiments are performed in Section \ref{sec:exp} demonstrating the superior properties of the Fisher-Rao norm.

%%%%%%%%%%%%%%%%%%%%%%%%%%%%%%%%%%%%%%%%%%%%%%%%%%%%%
\section{Geometry of Deep Rectified Networks}\label{sec:geometry}
\begin{defn}
	\label{def:mlp}
	\rm
	The function class $\mathcal{H}_L$ realized by the \textit{feedforward neural network architecture} of depth $L$ with coordinate-wise activation functions $\sigma_l$ is defined as the set of functions $f_\theta : \mathcal{X} \to \mathcal{Y}$ ($\mathcal{X} \subseteq \mathbb{R}^p, \mathcal{Y} \subseteq \mathbb{R}^K$)\footnote{It is possible to generalize the above architecture to include linear pre-processing operations such as zero-padding and average pooling.} with
	$
	f_\theta(x) = \sigma_{L+1}( \sigma_{L}(\ldots \sigma_2 (\sigma_1( x^T W^0) W^1 ) W^2) \ldots) W^{L})
	$
	where $\theta \in \Theta_L \subseteq \mathbb{R}^d$ ($d = p k_1 + \sum_{i=1}^{L-1} k_i k_{i+1} + k_L  K$) and
	$
	\Theta_L = \mathbb{R}^{p \times k_1} \times \mathbb{R}^{k_1 \times k_2} \times \ldots \times \mathbb{R}^{k_{L-1}\times k_{L}} \times \mathbb{R}^{k_L \times K}$.
\end{defn}
For simplicity of calculations, we have set all bias terms to zero\footnote{In practice, we found that setting the bias to zero does not significantly impact results on image classification tasks such as MNIST and CIFAR-10.}. Also, as pointed out by \cite{ferencblog}, a bias-less network with homogeneous coordinates (in the first layer) can be nearly as powerful as one with biases in terms of the functions it can model. We also assume throughout the paper that 
		$\sigma(z) = \sigma'(z) z$
for all the activation functions, which includes ReLU $\sigma(z)=\max\{0,z\}$, ``leaky'' ReLU $\sigma(z)=\max\{\alpha z,z\}$, and linear activations as special cases. 

To make the exposition of the structural results concise, we define the following intermediate functions. The output value of the $t$-th layer hidden node is denoted as $O^t(x) \in \mathbb{R}^{k_t}$, and the corresponding input value as $N^{t}(x) \in \mathbb{R}^{k_t}$, with $O^t(x) = \sigma_{t} (N^{t}(x))$. By definition, $O^{0}(x) = x^T \in \mathbb{R}^p$, and the final output $O^{L+1}(x) = f_\theta(x) \in \mathbb{R}^K$. The subscript $i$ on $N^t_i, O^t_i$ denotes the $i$-th coordinate of the respective vector. 

Given a loss function $\ell(\cdot, \cdot)$, the statistical learning problem can be phrased as optimizing the unobserved population loss: 
$
	L(\theta) := \E_{(X, Y)\sim \mathcal{P}} \ell(f_{\theta}(X), Y)
$
based on i.i.d. samples $\{ (X_i, Y_i) \}_{i=1}^N$ from the unknown joint distribution $\mathcal{P}$. The unregularized empirical objective function is denoted by 
$
	\widehat{L}(\theta) := \widehat{\E}  \ell(f_{\theta}(X), Y) = \frac{1}{N}\sum_{i=1}^N \ell(f_{\theta}(X_i), Y_i) \enspace .
$

We first establish the following structural result for neural networks. 
It will be clear in the later sections that the lemma is motivated by the study of the Fisher-Rao norm, formally defined in Definition \ref{def:fr} below, and by information geometry. For the moment, however, let us provide a different viewpoint. For linear functions $f_\theta(x) = \langle \theta,x\rangle$, we clearly have that $\langle \partial f/\partial \theta, \theta \rangle = f_\theta(x)$. Remarkably, a direct analogue of this simple statement holds for neural networks, even if over-parametrized.

\begin{lem}[Structure in Gradient]
	\label{lem:induction}
	Given a single data input $x \in \mathbb{R}^p$, consider the feedforward neural network in Definition~\ref{def:mlp} with activations satisfying $\sigma(z) = \sigma'(z) z$. 
	Then for any $0\leq t \leq  s \leq L$, one has the identity
	$
		\sum_{i \in [k_t], j \in [k_{t+1}]} \frac{\partial O^{s+1}}{\partial W^{t}_{ij}} W^{t}_{ij} = O^{s+1}(x)
	$.
	In addition, it holds that
	$$
		\sum_{\substack{i \in [k_t], j \in [k_{t+1}], \\ 0\leq t \leq L}} \frac{\partial O^{L+1}}{\partial W^{t}_{ij}} W^{t}_{ij} = (L+1) O^{L+1}(x) \enspace .
	$$
\end{lem}
Lemma~\ref{lem:induction} reveals the structural constraints in the gradients of rectified networks. In particular, even though the gradients lie in an over-parametrized high-dimensional space, many equality constraints are induced by the network architecture. Before we unveil the surprising connection between Lemma~\ref{lem:induction} and the proposed Fisher-Rao norm, let us take a look at an immediate corollary of this result. The following corollary establishes a large-margin property of stationary points that separate the data. 

\begin{cor}[Large Margin Stationary Points]\label{cor:largemargin}
	Consider the binary classification problem with $\mathcal{Y} = \{-1, +1\}$, and a neural network where the output layer has only one unit. Choose the hinge loss $\ell(f, y) = \max\{0,1 - yf\}$. If a certain parameter $\theta$ satisfies two properties: (a) $\theta$ is a stationary point for $\widehat{L}(\theta)$: $\nabla_\theta \widehat{L}(\theta) = \mathbf{0}$; (b) $\theta$ separates the data: $Y_i f_{\theta}(X_i)>0 : \forall i\in[N]$,
	then it must be that $\theta$ is a large margin solution: $Y_i f_{\theta}(X_i)\geq 1 : \forall i\in[N]$.
	The same result holds for the population criteria $L(\theta)$, in which case if condition $(b)$ holds $\mathbb{P}(Yf_{\theta}(X) >0) = 1$, then $\mathbb{P}(Yf_{\theta}(X) \geq 1) = 1$.
\end{cor}
Granted, the above corollary can be proved from first principles without the use of Lemma~\ref{lem:induction}, but the proof reveals a quantitative statement about stationary points along arbitrary directions $\theta$.

The following corollary is another direct consequence of Lemma~\ref{lem:induction}.
\begin{cor}[Stationary Points for Deep Linear Networks]\label{cor:stationary}
	Consider linear neural networks with $\sigma(x) = x$ and square loss function. Then all stationary points that satisfy $\nabla_{\theta} \widehat{L}(\theta) =  \nabla_{\theta} \widehat{\E}  \left[ \frac{1}{2}(f_{\theta}(X)- Y)^2 \right] = 0 $ must also satisfy
	$\langle w(\theta), \mathbf{X}^T \mathbf{X} w(\theta) -  \mathbf{X}^T \mathbf{Y}  \rangle = 0$,
	where $w(\theta) = \prod_{t=0}^{L} W^t \in \mathbb{R}^p$, $\mathbf{X} \in \mathbb{R}^{N \times p}$ and $\mathbf{Y} \in \mathbb{R}^{N}$ are the data matrices.
\end{cor}

\begin{remark}
	\rm
	This simple Lemma is not quite asserting that all stationary points are global optima, since global optima satisfy
	$\mathbf{X}^T \mathbf{X} w(\theta) -  \mathbf{X}^T \mathbf{Y} = \mathbf{0}$,
	while we only proved that stationary points satisfy
	$\langle w(\theta), \mathbf{X}^T \mathbf{X} w(\theta) -  \mathbf{X}^T \mathbf{Y}  \rangle = 0$.
\end{remark}

\begin{remark}
\rm
Recursively applying Lemma \ref{lem:induction} yields a tower of derivative constraints. For instance $\langle \theta, \Hess_\theta (f_\theta) \theta \rangle = L(L+1)f_\theta$.
\end{remark}

%%%%%%%%%%%%%%%%%%%%%%%%%%%%%%%%%%%%%%%%%%%%%%%%%%%%%
\section{Fisher-Rao Norm and Geometry}
\label{sec:fr}

In this section, we propose a new notion of complexity of neural networks that can be motivated by geometrical invariance considerations, specifically the Fisher-Rao metric of information geometry. 
After describing geometrical motivation in Section~\ref{sec:info-geo} we define the Fisher-Rao norm and describe some of its properties. Detailed comparison with the known norm-based capacity measures and generalization results are delayed to Section~\ref{sec:norm-genrl}.  

\subsection{Motivation and invariance}
\label{sec:info-geo}
In this section, we will provide the original intuition and motivation for our proposed Fisher-Rao norm from the viewpoint of geometric invariance.

\noindent \textbf{Information geometry and the Fisher-Rao metric} \quad
Information geometry provides a window into geometric invariances when we adopt a generative framework where the data generating process belongs to a parametric family $\mathcal{P} \in \{ \mathcal{P}_\theta \, | \, \theta \in \Theta_L \}$ indexed by the parameters of the neural network architecture. The Fisher-Rao metric on $\{ \mathcal{P}_\theta\}$ is defined in terms of a local inner product for each value of $\theta \in \Theta_L$ as follows. For each $\alpha, \beta \in \mathbb{R}^d$ define the corresponding tangent vectors $\bar{\alpha} := \left. {\rm d}p_{\theta+ t\alpha} / {\rm d}t \right|_{t=0}$, $\bar{\beta} := \left. {\rm d}p_{\theta+ t\beta} / {\rm d}t\right|_{t=0}$. Then for all $\theta \in \Theta_L$ and $\alpha,\beta \in \mathbb{R}^d$ we define the local inner product
$
\langle \bar{\alpha}, \bar{\beta} \rangle_{p_{\theta}} :=
\int
\frac{\bar{\alpha}}{p_{\theta}}
\,
\frac{\bar{\beta}}{p_{\theta}}
\,
p_{\theta}
$,
The above inner product extends to a Riemannian metric on the space of positive densities called the Fisher-Rao metric\footnote{\cite{bauer2016uniqueness} showed that it is essentially the the unique metric invariant under the diffeomorphism group.}. The relationship between the Fisher-Rao metric and the Fisher information matrix $I_\theta$ in statistics literature follows from the identity
$
	\langle \bar{\alpha}, \bar{\beta} \rangle_{p_{\theta}} = \langle \alpha, I_\theta \beta \rangle
$.
Notice that the Fisher information matrix induces a \emph{semi}-inner product $(\alpha,\beta) \mapsto \langle \alpha, I_{\theta}\beta \rangle$ unlike the Fisher-Rao metric which is non-degenerate\footnote{The nullspace of $I_{\theta}$ maps to $\mathbf{0}$ under $\alpha \mapsto \left. {\rm d}p_{\theta+ t\alpha} / {\rm d}t \right|_{t=0}$.}.
 If we make the additional modeling assumption that $p_\theta(x,y) = p(x)p_\theta(y \, | x)$ then the Fisher information becomes
 $
 	I_{\theta} = \mathbb{E}_{(X,Y) \sim \mathcal{P}_\theta} \left[ \nabla_\theta \log p_\theta( Y \, | \, X) \otimes \nabla_\theta \log p_\theta( Y \, | \, X) \right]
 $.
If we now identify our loss function as $\ell(f_\theta(x),y) = - \log p_\theta(y \, | \,x)$ then the Fisher-Rao metric coincides with the Fisher-Rao norm when $\alpha = \beta = \theta$. In fact, our Fisher-Rao norm encompasses the Fisher-Rao metric and generalizes it to the case when the model is misspecified $\mathcal{P} \not\in \{\mathcal{P}_\theta\}$.

\noindent \textbf{Flatness} \quad
Having identified the geometric origin of Fisher-Rao norm, let us study the implications for generalization of flat minima. \cite{dinh2017sharp} argued by way of counter-example that the existing measures of flatness are inadequate for explaining the generalization capability of multi-layer neural networks. Specifically, by utilizing the invariance property of multi-layer rectified networks under non-negative nodewise rescalings, they proved that the Hessian eigenvalues of the loss function can be made arbitrarily large, thereby weakening the connection between flat minima and generalization. They also identified a more general problem which afflicts Hessian-based measures of generalization for any network architecture and activation function: the Hessian is sensitive to network parametrization whereas generalization should be invariant under general coordinate transformations.  
Our proposal can be motivated from the following fact
% \footnote{Set $\ell(f_\theta(x),y) = - \log p_\theta(y|x)$ and recall the fact that Fisher information can be viewed as variance as well as the curvature.}
which relates flatness to geometry (under appropriate regularity conditions)
$
\mathbb{E}_{(X,Y) \sim \mathcal{P}_\theta} \langle \theta, \Hess_\theta \left[ \ell(f_\theta(X), Y)\right] \, \theta \rangle = \Vert \theta \Vert_{\rm fr}^2
$.
In other words, the Fisher-Rao norm evades the node-wise rescaling issue because it is exactly invariant under linear re-parametrizations. The Fisher-Rao norm moreover possesses an ``infinitesimal invariance'' property under non-linear coordinate transformations, which can be seen by passing to the infinitesimal form where non-linear coordinate invariance is realized exactly by the infinitesimal line element,
$
 	{\rm d}s^2 = \sum_{i,j \in [d]}[I_{\theta}]_{ij} \, {\rm d}\theta_i {\rm d}\theta_j
 $.
Comparing with $\Vert \theta \Vert_{\rm fr}$ reveals the geometric interpretation of the Fisher-Rao norm as the approximate geodesic distance from the origin. It is important to realize that our definition of flatness differs from \citep{dinh2017sharp} who employed the Hessian loss $\Hess_\theta \big[\widehat{L}(\theta)\big]$. Unlike the Fisher-Rao norm, the norm induced by the Hessian loss does not enjoy the infinitesimal invariance property (it only holds at critical points).

\noindent \textbf{Natural gradient} \quad
There exists a close relationship between the Fisher-Rao norm and the natural gradient. In particular, the natural gradient descent is simply the steepest descent direction induced by the Fisher-Rao geometry of $\{\mathcal{P}_\theta\}$. Indeed, the natural gradient can be expressed as a semi-norm-penalized iterative optimization scheme as follows,
\begin{equation*}
	\label{e:ng}
    \theta_{t+1} = \argmin_{\theta \in \mathbb{R}^d} \left[\langle \theta-\theta_t, \nabla \widehat{L}(\theta_t) \rangle + \frac{1}{2\eta_t} \Vert \theta - \theta_t \Vert^2_{\mathbf{I}(\theta_t)} \right] \enspace ,
\end{equation*}
where the positive semi-definite matrix $\mathbf{I}(\theta_t)$ changes with different $t$. We emphasize that in addition to the invariance property of the natural gradient under re-parametrizations, there exists an ``approximate invariance'' property under over-parametrization, which is not satisfied for the classic gradient descent. The formal statement and its proof are deferred to Sec. \ref{sec:inv_nat_grad}.
The invariance property is desirable: in multi-layer ReLU networks, there are many equivalent re-parametrizations of the problem, such as nodewise rescalings, which may slow down the optimization process. The advantage of natural gradient is also illustrated empirically in Section~\ref{sec:exp}.

\subsection{An analytical formula}
\begin{defn}\label{def:fr}
The Fisher-Rao norm for a parameter $\theta$ is defined as the quadratic form $\| \theta \|_{\rm fr}^2 := \langle \theta, \mathbf{I}(\theta) \theta\rangle$
where $\mathbf{I}(\theta) = \E[\nabla_\theta \ell(f_\theta(X), Y) \otimes \nabla_\theta \ell(f_\theta(X), Y) ]$.
\end{defn}
The underlying distribution for the expectation in the above definition has been left ambiguous because it will be useful to specialize to different distributions depending on the context. Even though we call the above quantity the ``Fisher-Rao norm,'' it should be noted that it does not satisfy the triangle inequality. 
The following Theorem unveils a surprising identity for the Fisher-Rao norm.  

\begin{thm}[Fisher-Rao norm]
	\label{thm:FR}
	Assume the loss function $\ell(\cdot, \cdot)$ is smooth in the first argument. 
	The following identity holds for a feedforward neural network (Definition~\ref{def:mlp}) with $L$ hidden layers and activations satisfying $\sigma(z) = \sigma'(z) z$: 
	\begin{align}
		\| \theta \|_{\rm fr}^2 = (L+1)^2 \E 
		\left \langle \frac{\partial \ell(f_\theta(X), Y)}{\partial f_\theta(X) }, f_{\theta}(X)  \right \rangle^2 \enspace .
	\end{align}
\end{thm}
The proof of the Theorem relies mainly on the geometric Lemma~\ref{lem:induction} that describes the gradient structure of multi-layer rectified networks. 
\begin{remark}
\rm 
For absolute-value loss, the FR norm becomes proportional to the function space norm
$
	\| \theta \|_{\rm fr} = (L+1) \left(\mathbb{E} \, f_\theta(X)^2\right)^{1/2}
$.
Similarly for squared loss with residual modeled as\footnote{It also holds for other appropriate losses when the model residual follows from generalized linear models.} $Y|X \sim N(f_\theta(X), \sigma^2)$.
\end{remark}

Before illustrating how the explicit formula in Theorem~\ref{thm:FR} can be viewed as a unified ``umbrella'' for many of the known norm-based capacity measures, let us point out one simple invariance property of the Fisher-Rao norm, which follows as a direct consequence of Thm.~\ref{thm:FR}. This property is not satisfied for $\ell_2$ norm, spectral norm, path norm, or group norm.  
\begin{cor}[Invariance]
	If there are two parameters $\theta_1, \theta_2 \in \Theta_L$ such that
	they are equivalent, in the sense that $f_{\theta_1} = f_{\theta_2}$, then their Fisher-Rao norms are equal, i.e.,
	$
		\| \theta_1 \|_{\rm fr} = \| \theta_2 \|_{\rm fr}
	$.
\end{cor}

\subsection{Norms and geometry}

In this section we will employ Theorem~\ref{thm:FR} to reveal the relationship among different norms and their corresponding geometries. Norm-based capacity control is an active field of research for understanding why deep learning generalizes well,  including $\ell_2$ norm (weight decay) in \citep{krogh1992simple, krizhevsky2012imagenet}, path norm in \citep{neyshabur2015path}, group-norm in \citep{neyshabur2015norm}, and spectral norm in \citep{bartlett2017spectrally}. All these norms are closely related to the Fisher-Rao norm, despite the fact that they capture distinct inductive biases and different geometries. 

For simplicity, we will showcase the derivation with the absolute loss function $\ell(f, y) = |f - y|$ and when the output layer has only one node ($k_{L+1} = 1$). The argument can be readily adopted to the general setting. We will show that the Fisher-Rao norm serves as a lower bound for all the norms considered in the literature, with some pre-factor whose meaning will be clear in Section~\ref{sec:norm-compare}. In addition, the Fisher-Rao norm enjoys an interesting umbrella property: by considering a more constrained geometry (motivated from algebraic norm comparison inequalities) the Fisher-Rao norm motivates new norm-based capacity control methods.

The main theorem we will prove is informally stated as 
\begin{thm}[Norm comparison in Section~\ref{sec:norm-compare}, informal]\label{thm:normcomp}
	Denoting $\vvvert \cdot \vvvert$ as any one of: (1) spectral norm, (2) matrix induced norm, (3) group norm, or (4) path norm, we have
	$
		\frac{1}{L+1} \| \theta \|_{\rm fr} \leq  \vvvert \theta \vvvert
	$,
	for any $\theta \in \Theta_L = \{W^0, W^1, \ldots, W^{L}\}$. The specific norms (1)-(4) are formally introduced in Definitions~\ref{def:norm-spectral}-\ref{def:norm-induced}.
\end{thm}

The detailed proof of the above theorem will be the main focus of Section~\ref{sec:norm-compare}. Here we will give a sketch on how the results are proved.
For the absolute loss, one has $\big( \partial \ell(f_\theta(X), Y) / \partial f_\theta(X) \big)^2 = 1$ and therefore Theorem~\ref{thm:FR} simplifies to,
\begin{align}
	\label{eq:FR}
	\| \theta \|_{\rm fr}^2
	 = (L+1)^2 \E_{X \sim \mathcal{P}} \left[ v(\theta, X)^T XX^T v(\theta, X) \right] \enspace ,
\end{align}
where we have defined the product of matrices $v(\theta, x):= W^0 D^1(x) W^1 D^2(x) \cdots D^L(x) W^L D^{L+1}(x) \in \mathbb{R}^p$ and $D^t(x) = {\rm diag}[\sigma'(N^t(x))] \in \mathbb{R}^{k_t \times k_t}$, for $0<t \leq L+1$. The norm comparison results are thus established through a careful decomposition of the data-dependent vector $v(\theta, X)$, in distinct ways according to the comparing norm/geometry.

%%%%%%%%%%%%%%%%%%%%%%%%%%%%%%%%%%%%%%%%%%%%%%%%%%%%%
\section{Capacity Control and Generalization}
\label{sec:norm-genrl}

In this section, we discuss in full detail the questions of geometry, capacity measures, and generalization. First, let us define empirical \textit{Rademacher complexity} for the parameter space $\Theta_L$, conditioned on data $\{X_i : i\in [N] \}$, as 
	$
		\R_N(\Theta_L) = \E_{\epsilon} \sup_{\theta \in \Theta_L} \frac{1}{N} \sum_{i=1}^N \epsilon_i f_{\theta}(X_i) 
	$,
	where $\epsilon_i, i \in [N]$ are i.i.d. Rademacher random variables.

\subsection{Norm Comparison}
\label{sec:norm-compare}
Let us collect some definitions before stating each norm comparison result. For a vector $v$, the vector $\ell_p$ norm is denoted $\| v \|_p := \left( \sum_{i} |v_i|^p \right)^{1/p}$, $p>0$. For a matrix $M$, $\| M \|_{\sigma} := \max_{v \neq 0} \|v^T M \| / \| v \|$ denotes the spectral norm; $\| M \|_{p\rightarrow q} = \max_{v \neq 0} \|v^T M\|_q / \| v \|_p$ denotes the matrix induced norm, for $p, q\geq 1$; $\| M \|_{p, q} = \big[ \sum_{j} \big( \sum_{i} |M_{ij}|^p \big)^{q/p}\big]^{1/q}$ denotes the matrix group norm, for $p, q\geq 1$. 
Define the subset of parameters induced by the Fisher-Rao geometry,
$
	B_{\rm fr}(r) := \{ \theta \in \Theta_L : \| \theta \|_{\rm fr} \leq (L+1)r \}
$.
For any norm $\| \cdot \|$ defined on $\Theta_L$, let $B_{\| \cdot \|}(r) := \{\theta \in \Theta_L: \| \theta \|_{\| \cdot \|} \leq r \}$ denote the corresponding ball of radius $r$, centered at the origin.

We will consider the spectral norm $\| \theta \|_\sigma$, group norm $\| \theta \|_{p,q}$, matrix induced norm $\| \theta \|_{p \to q}$ and path norm $\| \pi(\theta) \|_q$, where $\| \theta \|_\sigma := \prod_{t=0}^{L} \| W^t \|_{\sigma}$, $\| \theta \|_{p,q} := \prod_{t=0}^{L} \| W^t \|_{p,q}$, $\| \theta \|_{p \to q} := \prod_{t=0}^{L} \| W^t \|_{p \to q}$, $\| \pi(\theta) \|_q := \big(\sum_{i_0, i_1, \ldots, i_L}\prod_{t=0}^L |W_{i_t i_{t+1}}^t|^q \big)^{1/q}$. In addition, for any chain $P = (p_0, p_1, \ldots, p_{L+1}), p_i > 0$ we define the chain of induced norm $\Vert \theta \Vert_P : =  \prod_{t=0}^{L} \| W^t \|_{p_t \rightarrow p_{t+1}}$.

% \smallskip
% \subsubsection{\textbf{Spectral norm.}}
\begin{defn}[Spectral norm]
	\label{def:norm-spectral}
	\rm
	Define the following data-dependent ``spectral norm'':
	\begin{align}\label{eq:norm-spectral}
		\vvvert \theta  \vvvert_{\sigma} & := \left[ \mathbb{E}\left( \| X \|^2  \prod_{t=1}^{L+1} \| D^{t}(X)\|^2_{\sigma} \right) \right]^{1/2} \| \theta \|_\sigma \enspace . 
	\end{align}
\end{defn}

\begin{remark}
	\rm
	\label{remark:spectral}
	Spectral norm as a capacity control has been considered in \citep{bartlett2017spectrally}. Theorem~\ref{thm:normcomp} shows that spectral norm serves as a more stringent constraint than Fisher-Rao norm. Let us provide an explanation of the pre-factor $\big[ \mathbb{E}\big( \| X \|^2  \prod_{t=1}^{L+1} \| D^{t}(X)\|^2_{\sigma} \big) \big]^{1/2}$ here. Applying Theorem~\ref{thm:normcomp} to \eqref{eq:norm-spectral}, with the expectation over the empirical measure $\widehat{\E} $, then, because $\| D^{t}(X)\|_{\sigma} \leq 1$, we obtain for $1/r = [\widehat{\E}\Vert X \Vert^2]^{1/2}$, that
	$
	B_{\| \cdot \|_\sigma}(r) \subset B_{\rm fr}(1)
	$.
	It follows from Theorem 1.1 in \citep{bartlett2017spectrally} that a subset of the $B_{\rm fr}(1)$ characterized by the \textit{spectral ball} of radius $r = [\widehat{\E}\Vert X \Vert^2]^{-1/2}$ enjoys the following upper bound on Rademacher complexity under mild conditions:
	$\R_N\left(B_{\Vert \cdot \Vert\sigma}(r) \right) \precsim {\rm Polylog} \,  / \sqrt{N}  \rightarrow 0$.
	Interestingly, the additional factor $[ \widehat{\E}  \| X \|^2 ]^{1/2}$ in Theorem 1.1 in \citep{bartlett2017spectrally} exactly cancels with our pre-factor in the norm comparison. 
	The above calculations show that a subset of $B_{\rm fr}(1)$, induced by the spectral norm geometry, has good generalization error. 
\end{remark}

% \smallskip
% \subsubsection{\textbf{Group norm.}}
\begin{defn}[Group norm]
	\label{def:norm-group}
	\rm
	Define the following data-dependent ``group norm'', for $p \geq 1, q>0$
	$$
		\vvvert \theta  \vvvert_{p,q} := \left[ \mathbb{E}\left( \| X \|_{p^*}^2  \prod_{t=1}^{L+1} \| D^{t}(X)\|_{q \rightarrow p^*}^2 \right)  \right]^{1/2} \| \theta  \|_{p,q}  \enspace ,
	$$
	where $1/p+ 1/p^{*} = 1$. Here $\| \cdot \|_{q \rightarrow p^*}$ denotes the matrix induced norm.
\end{defn}

\begin{remark}
	\label{remark:group}
	\rm
	Group norm as a capacity measure has been considered in \citep{neyshabur2015norm}. The same reasoning as before shows that group norm serves as a more stringent constraint than Fisher-Rao norm. %Again, let us  provide an explanation of the pre-factor $\big[ \mathbb{E}\big( \| X \|_{p^*}^2  \prod_{t=1}^{L+1} \| D^{t}(X)\|_{q \rightarrow p^*}^2 \big)  \big]^{1/2}$ here. 
	In particular, Theorem~\ref{def:norm-induced} implies that the group norm ball with radius defined by $1/r = (k^{[1/p^{*} - 1/q]_{+}})^L \max_i \| X_i \|_{p^*}$ is contained in the unit Fisher-Rao ball, 
	$
	B_{\| \cdot \|_{p,q}}(r) \subset B_{\rm fr}(1)
	$.
	By Theorem 1 in \citep{neyshabur2015norm} we obtain,
	$
	\R_N\left(B_{p,q}(r) \right)\precsim 2^L \cdot {\rm Polylog} \,  / \sqrt{N}  \rightarrow 0
	$.
	Once again, we point out that the intriguing combinatorial factor $( k^{[1 / p^{*} - 1 / q]_{+}})^L \max_i \| X_i \|_{p^*}$  in Theorem 1 of \cite{neyshabur2015norm} exactly cancels with our pre-factor in the norm comparison. 
		The above calculations show that another subset of $B_{\rm fr}(1)$, induced by the group norm geometry, has good generalization error (without additional factors).

\end{remark}

% \smallskip
% \subsubsection{\textbf{Path norm.}}
\begin{defn}[Path norm]
	\label{def:norm-path}
	\rm
	Define the following data-dependent ``path norm'', for $q\geq 1$
	{\small
	$$
		\vvvert \pi(\theta) \vvvert_{q} := \left[ \E \left( \sum_{ i_0,\ldots,i_L} \left|X_{i_0} \prod_{t=1}^{L+1} D^t_{i_t}(X)\right|^{q^*} \right)^{2/q^*} \right]^{1/2} \| \pi(\theta) \|_q 
	$$
	}where $\frac{1}{q} + \frac{1}{q^*} = 1$, indices set $i_0 \in [p], i_1 \in [k_1], \ldots i_L \in [k_L], i_{L+1} = 1$. Here $\pi(\theta)$ is a notation for all the paths (from input to output) of the weights $\theta$. 
\end{defn}

\begin{remark}
	\rm
	\label{remark:path}
 	The path norm $\Vert \pi(\cdot) \Vert_q$ has been investigated in \citep{neyshabur2015path}. Focusing on the case $q = 1$, Theorem \ref{thm:normcomp} gives $B_{\| \pi(\cdot) \|_1}(r) \subset B_{\rm fr}(1)$ for $1/r = \max_i \| X_i \|_\infty$. By Corollary 7 in \citep{neyshabur2015norm} we obtain
	$
	\R_N\left( B_{\Vert \pi(\cdot) \Vert_1}(r) \right) \precsim 2^L \cdot {\rm Polylog} \, / \sqrt{N} \rightarrow 0
	$.
	Once again, the additional factor appearing in the Rademacher complexity bound in \citep{neyshabur2015norm}, cancels with our pre-factor in the norm comparison. 
\end{remark}

% \smallskip
% \subsubsection{\textbf{Matrix induced norm}.}
\begin{defn}[Induced norm]
	\label{def:norm-induced}
	\rm
	Define the following data-dependent ``matrix induced norm'', for $p, q>0$, as
	$$
		\vvvert \theta  \vvvert_{p \rightarrow q} := \left[ \mathbb{E}\left( \| X \|_{p}^2  \prod_{t=1}^{L+1} \| D^{t}(X)\|^2_{q \rightarrow p} \right) \right]^{1/2} \| \theta \|_{p \rightarrow q} \enspace .
	$$
\end{defn}

Remark that $\| D^{t}(X)\|_{q \rightarrow p}^2$ may contain dependence on $k$ when $p\neq q$. This motivates us to consider the following generalization of matrix induced norm, where the norm for each $W^t$ can be different.
\begin{defn}[Chain of induced norm]\label{def:chain}
	\rm
	Define the following data-dependent ``chain of induced norm'', for a chain of $P = (p_0, p_1, \ldots, p_{L+1}), p_i > 0$
	\begin{align*}
		\vvvert \theta  \vvvert_{P} := \left[ \mathbb{E}\left( \| X \|_{p_0}^2  \prod_{t=1}^{L+1} \| D^{t}(X)\|^2_{p_t \rightarrow p_t} \right) \right]^{1/2} \| \theta \|_P \enspace .
	\end{align*}
\end{defn}

\begin{remark}
	\rm
	Theorem~\ref{thm:normcomp} applied to Definition~\ref{def:chain} exhibits a new flexible norm that dominates the Fisher-Rao norm. The example shows that one can motivate a variety of new norms (and their corresponding geometry) as subsets of the Fisher-Rao norm ball. 
\end{remark}

We will conclude this section with two geometric observations about the Fisher-Rao norm with absolute loss function $\ell(f, y) = |f - y|$ and one output node. In this case, even though $B_{\rm fr}(1)$ is non-convex, it is star-shaped. 

\begin{lem}[Star shape]
	\label{lem:star-shape}
	For any $\theta \in \Theta_L$, let $\{r \theta, r > 0\}$ denote the line connecting between $0$ and $\theta$ to infinity.
	Then one has,
	$
		\frac{d}{dr} \|r \theta \|_{\rm fr}^2 = \frac{2(L+1)}{r} \|r \theta \|_{\rm fr}^2
	$
	which also implies
	$
		\|r \theta \|_{\rm fr} = r^{L+1} \| \theta \|_{\rm fr}
	$.
\end{lem}
Despite the non-convexity of $B_{\rm fr}(1)$, there is certain convexity in the function space: 
\begin{lem}[Convexity in $f_{\theta}$]
	\label{lem:convex.f}
	For any $\theta_1, \theta_2 \in \Theta_L$ such that $\frac{1}{L+1}\| \theta_1 \|_{\rm fr},\frac{1}{L+1}\| \theta_2 \|_{\rm fr} \leq 1$ we have for any $0 < \lambda <1$, the convex combination
	$\lambda f_{\theta_1} + (1-\lambda) f_{\theta_2}$ can be realized by a parameter $\theta' \in \Theta_{L+1}$ in the sense
	$
	f_{\theta'} = \lambda f_{\theta_1} + (1-\lambda) f_{\theta_2}
	$,
	and satisfies
	$
	\frac{1}{(L+1)+1} \| \theta' \|_{\rm fr} \leq 1
	$.
\end{lem}

\subsection{Generalization}
\label{sec:generalization}

In this section, we will investigate the generalization puzzle for deep learning through the lens of the Fisher-Rao norm. We will first introduce a simple proof in the case of multi-layer linear networks, that capacity control with Fisher-Rao norm ensures good generalization. Then we will provide an argument bounding the generalization error of rectified neural networks with Fisher-Rao norm as capacity control, via norm caparisons in Section~\ref{sec:norm-compare}. We complement our argument with extensive numerical investigations in Section~\ref{sec:exp}.

\begin{thm}[Deep Linear Networks]
	\label{thm:generalization-linear}
	Consider multi-layer linear networks with $\sigma(x) = x$, $L$ hidden layers, input dimension $p$ and single output unit, and parameters $\Theta_L = \{ W^0, W^1, \ldots, W^L \}$. Then we have
	$\E \R_N\left( B_{\rm fr}(\gamma)  \right) \leq \gamma \sqrt{p/N}$
	assuming the Gram matrix $\E[X X^T] \in \mathbb{R}^{p \times p}$ is full rank.
\end{thm}

\begin{remark}
	\rm
	Combining the above Theorem with classic symmetrization and margin bounds \citep{koltchinskii2002empirical}, one can deduce that for binary classification, the following generalization guarantee holds (for any margin parameter $\alpha>0$)
	$
		\mathbb{E} \mathbf{1}\left[ f_\theta(X) Y < 0  \right] \leq \widehat{\mathbb{E}}\mathbf{1}\left[ f_\theta(X) Y \leq \alpha  \right]  + \frac{C}{\alpha}  \R_N\left( B_{\rm fr}(\gamma)  \right) + C \sqrt{(1/N)\log 1/\delta}
	$
	for any $\theta \in B_{\rm fr}(\gamma)$ with probability at least $1 - \delta$, where $C>0$ is some constant. 
	We would like to emphasize that to explain generalization in this over-parametrized multi-layer linear network, it is indeed desirable that the generalization error in Theorem~\ref{thm:generalization-linear} only depends on the Fisher-Rao norm and the intrinsic input dimension $p$, without additional dependence on other network parameters (such as width, depth) or extraneous $X$-dependent factors. 
	
\end{remark}

In the case of ReLU networks, it turns out that bounding $\R_N\left( B_{\rm fr}(\gamma)  \right)$ directly in terms of the Fisher-Rao norm is a challenging task. Instead, we decompose $\R_N\left( B_{\rm fr}(\gamma)  \right)$ into two terms: the Rademacher complexity of a subset of the Fisher-Rao norm ball induced by distinct geometry (spectral, group, and path norm ball), plus a deterministic function approximation error term. 
Denote by $\F_{\rm fr}(\gamma):= \{ f_{\theta}: \theta \in B_{\rm fr}((L+1)\gamma)  \}$ the functions induced by parameters in $\gamma$-radius Fisher-Rao norm ball. Let the function class realized by spectral norm ball, group-$p, q$ norm ball, and path-$q$ norm ball be
	$\F_{\sigma}(\gamma) := \{ f_{\theta}: \theta \in B_{{\Vert \cdot \Vert}_\sigma}(\gamma) \}$,
	$\F_{p, q}(\gamma) := \{ f_{\theta}: \theta \in B_{\Vert \cdot \Vert_{p,q}}(\gamma) \} $,
	$\F_{\pi, q}(\gamma) := \{ f_{\theta}: \theta \in B_{\Vert \pi(\cdot) \Vert_{q}}(\gamma) \}$. 

As discussed in Remarks \ref{remark:spectral}-\ref{remark:path}, the spectral, group, and path norms induce distinct geometric subsets of the Fisher-Rao norm ball, in the following sense
	$\F_{\rm fr}(1) \supseteq \F_{\sigma}(r), \F_{p, q}(r), \F_{\pi, 1}(r)$ for appropriate radii $r$.
The following Theorem quantifies the generalization error of $\F_{\rm fr}(1)$, relying on our norm comparison inequality, and the results in current literature \citep{bartlett2017spectrally,neyshabur2015norm}.
\begin{prop}[Deep Rectified Networks]
	\label{thm:generalization-rectified}
	% Denote $\| f\|_2 := (\E_{X \sim \mathcal{P}_X} f^2(X))^{1/2}$ denotes the $\ell_2$ norm w.r.t $X$ distribution. Then
	For $\mathcal{G}$ taken as $\F_{\sigma}(r)$ with $r=[\widehat{\E}\Vert X \Vert^2]^{-1/2}$, one has
	{\small
	\begin{align*}
		\E \R_N\left( B_{\rm fr}(1)  \right) \leq \overbrace{\sup_{f \in \F_{\rm fr}(1)}\inf_{g \in \mathcal{G}(r)} \| f - g \|_{\infty} }^{\text{function approximation error}} + \frac{1}{\sqrt{N}} \cdot {\rm Polylog}.
	\end{align*}}Similar bounds hold for $\mathcal{G}$ taken as either $\F_{p,q}$ or $\F_{\pi,1}$.
\end{prop}
We would like to emphasize that the approximation error is on the function space rather than the parameter space. Furthermore, because of the cancellation of pre-factors, as discussed earlier, the generalization bound does \emph{not involve pre-factors}, in contrast to what one would get with a direct application of \citep{bartlett2017spectrally,neyshabur2015norm}.
% Note we have been making all non-logarithmic dependence on \textbf{depth} and \textbf{width} clear in the Rademacher complexity bound.
Before concluding this section, we present the contour plot of Fisher-Rao norm and path-$2$ norm in a simple two layer ReLU network in Fig.~\ref{fig:contour} (in Appendix), to better illustrate the geometry of Fisher-Rao norm and the subsets induced by other norms. We choose two weights as $x,y$-axis and plot the levelsets of the norms.

%%%%%%%%%%%%%%%%%%%%%%%%%%%%%%%%%%%%%%%%%%%%%%%%%%%%%
\section{Experiments}
\label{sec:exp}
\noindent \textbf{Over-parametrization with Hidden Units} \quad
In order to understand the effect of network over-parametrization we investigated the relationship between different proposals for capacity control and the number of parameters of the neural network. For simplicity we focused on a fully connected architecture consisting of $L$ hidden layers with $k$ neurons per hidden layer so that the expression simplifies to $d = k[p + k(L-1) + K]$. The network parameters were learned by minimizing the cross-entropy loss on the CIFAR-10 image classification dataset with no explicit regularization nor data augmentation. The cross-entropy loss was optimized using  200 epochs of minibatch gradient descent utilizing minibatches of size 50 and otherwise identical experimental conditions described in \citep{zhang2016understanding}. The same experiment was repeated using an approximate form of natural gradient descent called the Kronecker-factored approximate curvature (K-FAC) method \citep{martens2015optimizing} with the same learning rate and momentum schedules. The first fact we observe is that the Fisher-Rao norm remains approximately constant (or decreasing) when the network is overparametrized by increasing the width $k$ at fixed depth $L=2$ (see Fig.~\ref{fig:k}). If we vary the depth $L$ of the network at fixed width $k=500$ then we find that the Fisher-Rao norm is essentially constant when measured in its `natural units' of $L+1$ (Fig.~\ref{fig:l} supplementary material).  Finally, if we compare each proposal based on its absolute magnitude, the Fisher-Rao norm is distinguished as the minimum-value norm, and becomes $O(1)$ when evaluated using the model distribution. This self-normalizing property can be understood as a consequence of the relationship to flatness discussed in section \ref{sec:info-geo}, which holds when the expectation is taken with respect to the model. 

\noindent \textbf{Corruption with Random Labels} \quad
Over-parametrized neural networks tend to exhibit good generalization despite perfectly fitting the training set \citep{zhang2016understanding}. In order to pinpoint the ``correct'' notion of complexity which drives generalization error, we conducted a series of experiments in which we changed both the network size and the signal-to-noise ratio of the datasets.
In particular, we focus on the set of neural architectures obtained by varying the hidden layer width $k$ at fixed depth $L=2$ and moreover for each training/test example we assign a random label with probability $\alpha$. 

It can be seen from the last two panels of Fig.~\ref{fig:momentum} that for non-random labels ($\alpha = 0$), the empirical Fisher-Rao norm actually decreases with increasing $k$, in tandem with the generalization error and moreover this correlation seems to persist when we vary the label randomization. Overall the Fisher-Rao norm is distinguished from other measures of capacity by the fact that its empirical version seems to track the generalization gap and moreover this trend does not appear to be sensitive to the choice of optimization. The stability of the Fisher-Rao norm with respect to increasing $k$ suggests that the infinitely wide limit $k \to \infty$ exists and is independent of $k$, and indeed this was recently verified using mean-field techniques \cite{karakida2018universal}. Finally, we note that unlike the vanilla gradient, the natural gradient differentiates the different architectures by their Fisher-Rao norm (Fig.~\ref{fig:kfac} supplementary material). Although we don't completely understand this phenomenon, it is likely a consequence of the fact that the natural gradient is iteratively minimizing the FR semi-norm.

\noindent \textbf{Margin Story} \quad 
\cite{bartlett2017spectrally} adopted the margin story to explain generalization. They investigated the spectrally-normalized margin to explain why CIFAR-10 with random labels is a harder dataset (poorer generalization) than the uncorrupted CIFAR-10 (which generalize well). Here we adopt the same idea in this experiment, where we plot margin normalized by the empirical Fisher-Rao norm, in comparison to the spectral norm, based on the model trained either by vanilla gradient and natural gradient. It can be seen in the supplementary material that the Fisher-Rao-normalized margin also accounts for the generalization gap between random and original CIFAR-10. In addition, Table \ref{table:margins} shows that the empirical Fisher-Rao norm improves the normalized margin relative to the spectral norm. These results were obtained by optimizing with the natural gradient but are not sensitive to the choice of optimizer.

\begin{table}
\begin{center}
\begin{tabular}{@{}l|lll@{}} 
 \toprule
    & Model FR & Empirical FR & Spectral  \\
    \midrule
    $\alpha = 0$ & 1.61 & 22.68 & 136.67 \\
    $\alpha = 1$ & 2.12 & 35.98 & 205.56 \\
    Ratio & 0.76 & \textbf{0.63} & 0.66 \\
    \bottomrule
    \end{tabular}
    \end{center}
    \caption{
    Comparison of Fisher-Rao norm and spectral norm after training with natural gradient using original dataset ($\alpha=0$) and with random labels ($\alpha=1$). Qualitatively similar results holds for GD+momentum.
    }
    \label{table:margins}
\end{table}

\noindent \textbf{Natural Gradient and Pre-conditioning} \quad
% \label{sec:failures}
It was shown in \citep{shalev2017failures} that multi-layer networks struggle to learn certain piecewise-linear curves because the problem instances are poorly-conditioned. The failure was attributed to the fact that simply using a black-box model without a deeper analytical understanding of the problem structure could be computationally sub-optimal. Our results (in Fig.~\ref{fig:failures} of Appendix) suggest that the problem can be overcome within the confines of black-box optimization by using natural gradient. In other words, the natural gradient automatically pre-conditions the problem and appears to achieve  similar performance as that attained by hard-coded convolutions \citep{shalev2017failures}, within the same number of iterations.

\begin{figure}[!h]
\centering
\includegraphics[width=0.5\linewidth]{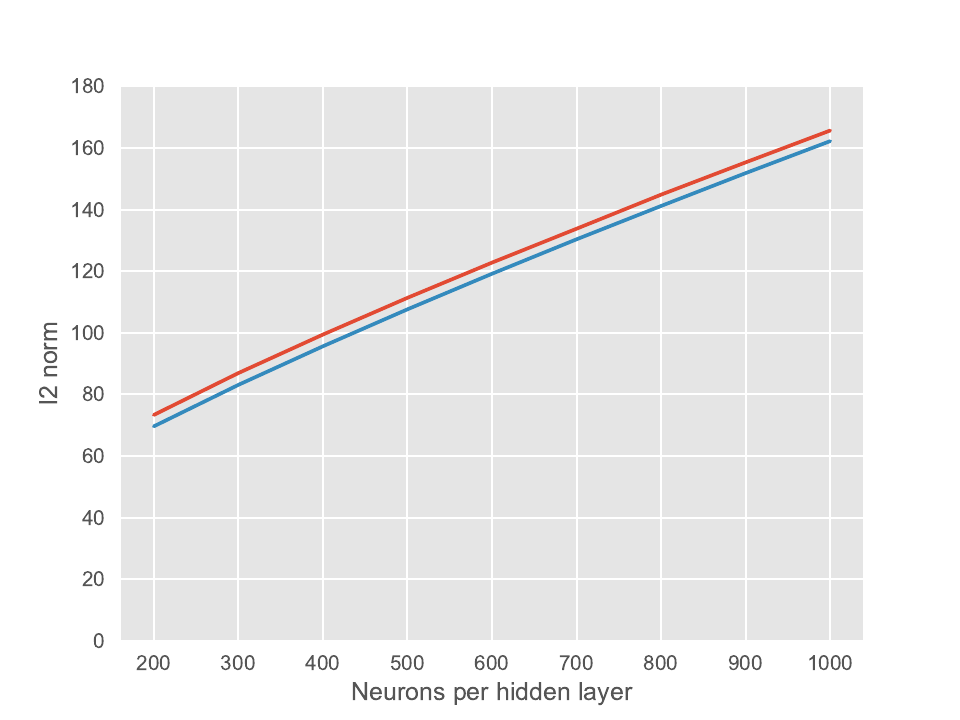}
\hspace{-0.5cm}
\includegraphics[width=0.5\linewidth]{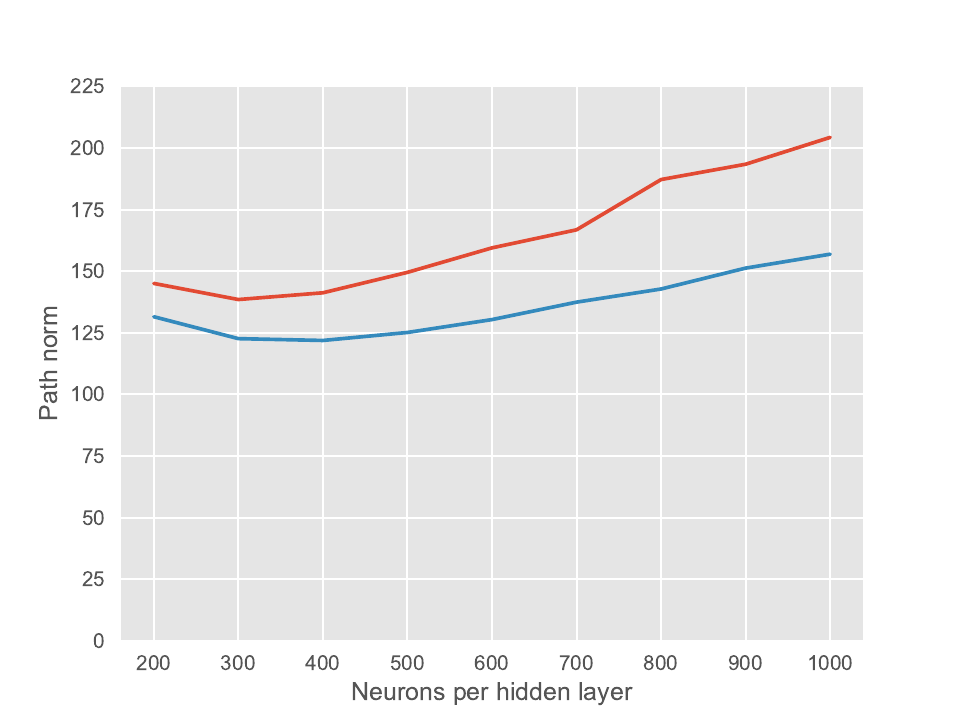}
\hspace{-0.5cm}
\includegraphics[width=0.5\linewidth]{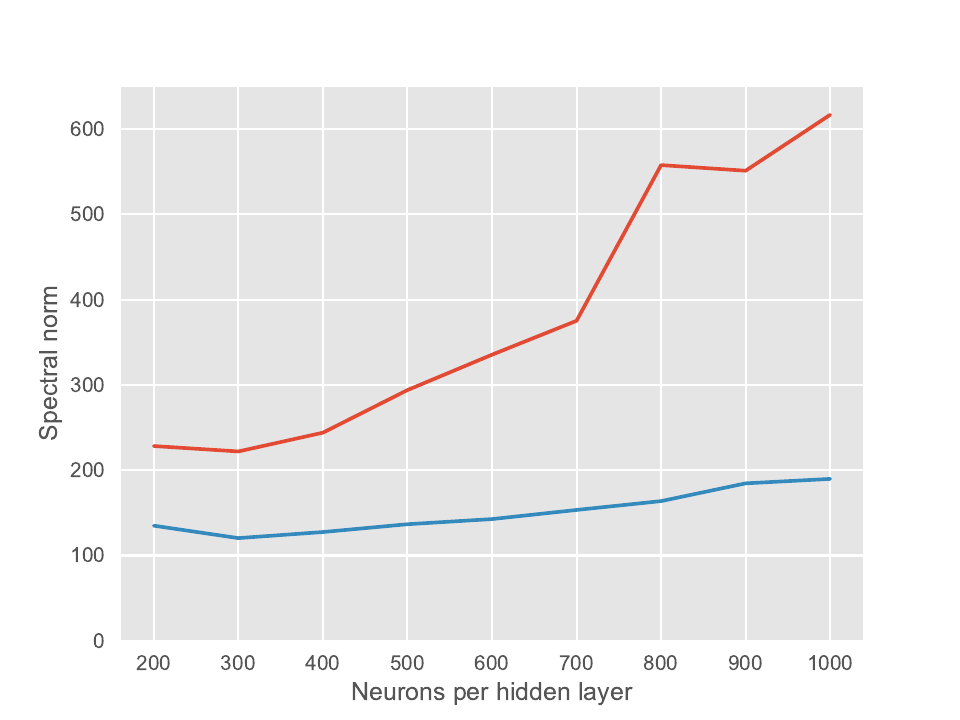}
\hspace{-0.5cm}
\includegraphics[width=0.5\linewidth]{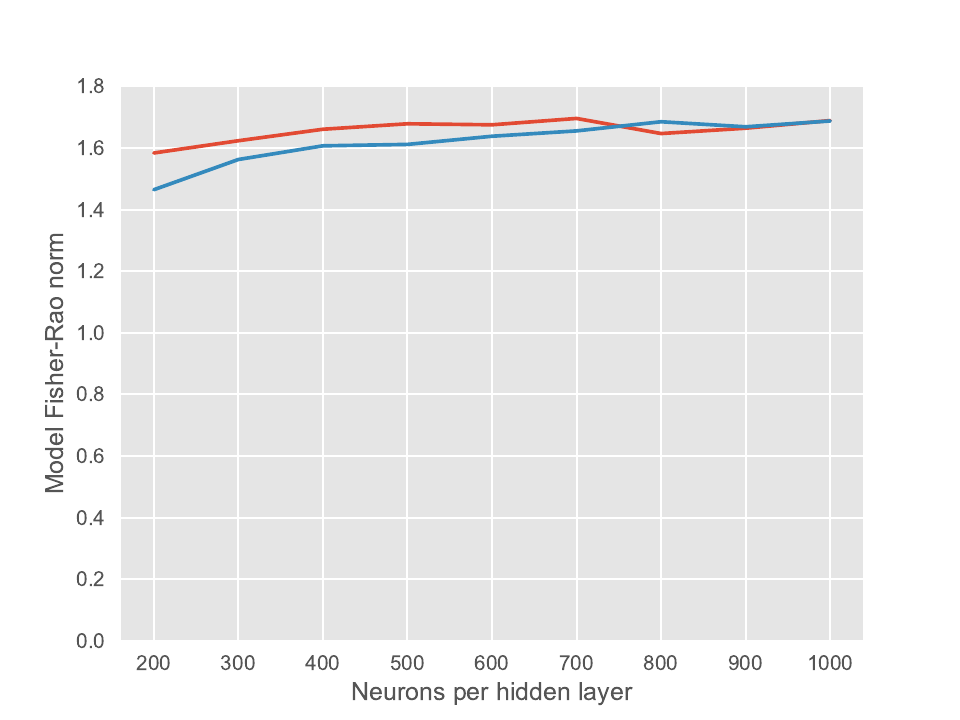}
\caption{Dependence of different norms on width $k$ of hidden layers $(L=2)$ after optimizing with vanilla gradient descent (red) and natural gradient descent (blue).} \label{fig:k}
\end{figure}

\begin{figure}[!h]
\centering
\includegraphics[width=0.5\linewidth]{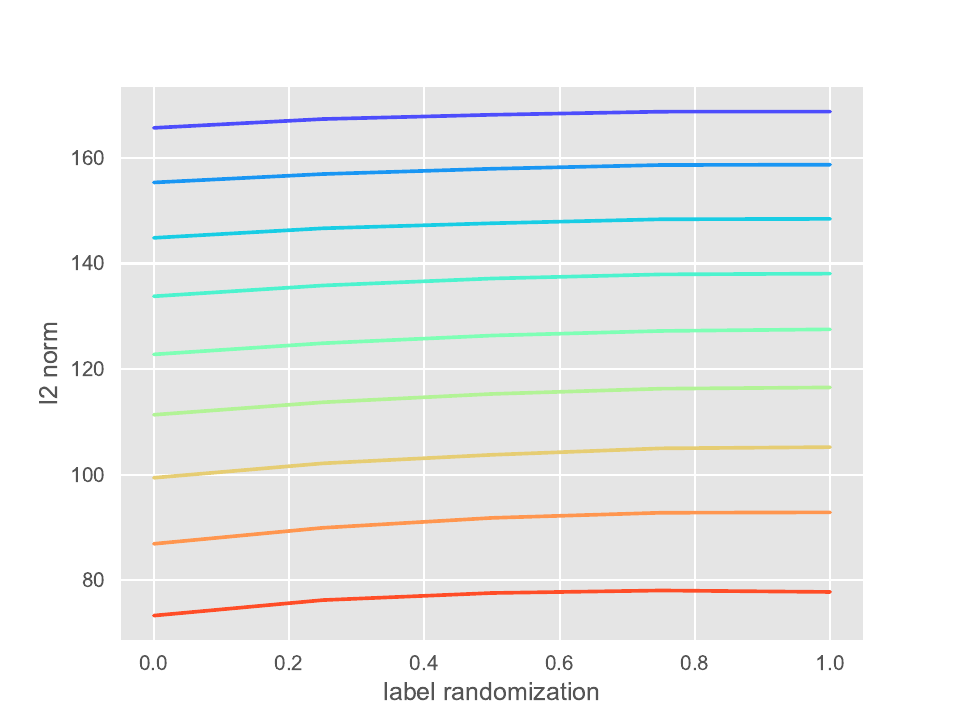}
\hspace{-0.5cm}
\includegraphics[width=0.5\linewidth]{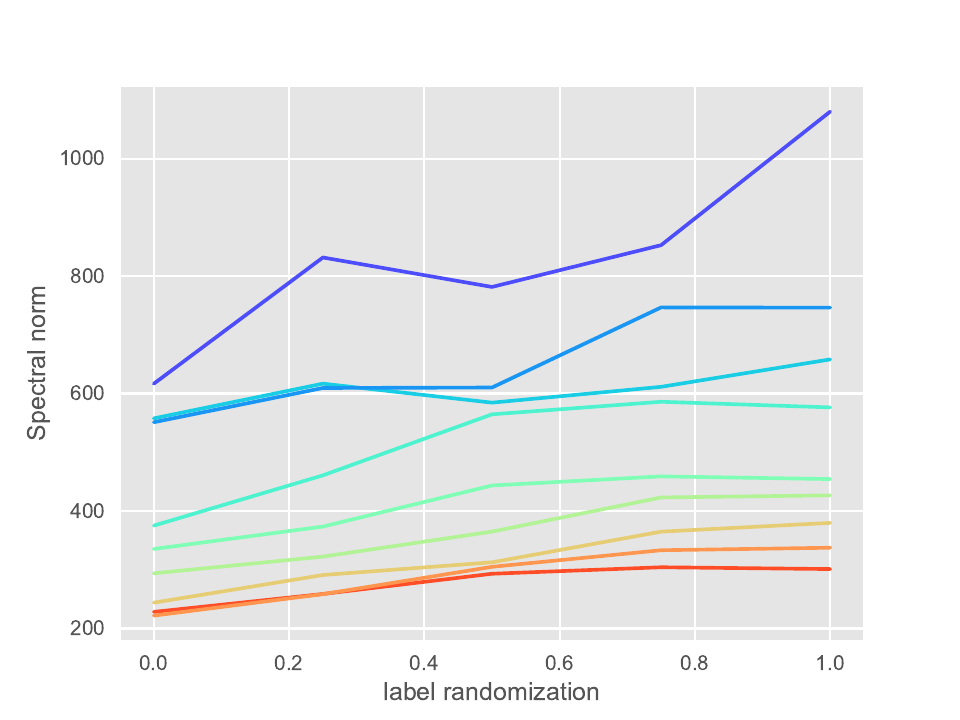}
\hspace{-0.5cm}
\includegraphics[width=0.5\linewidth]{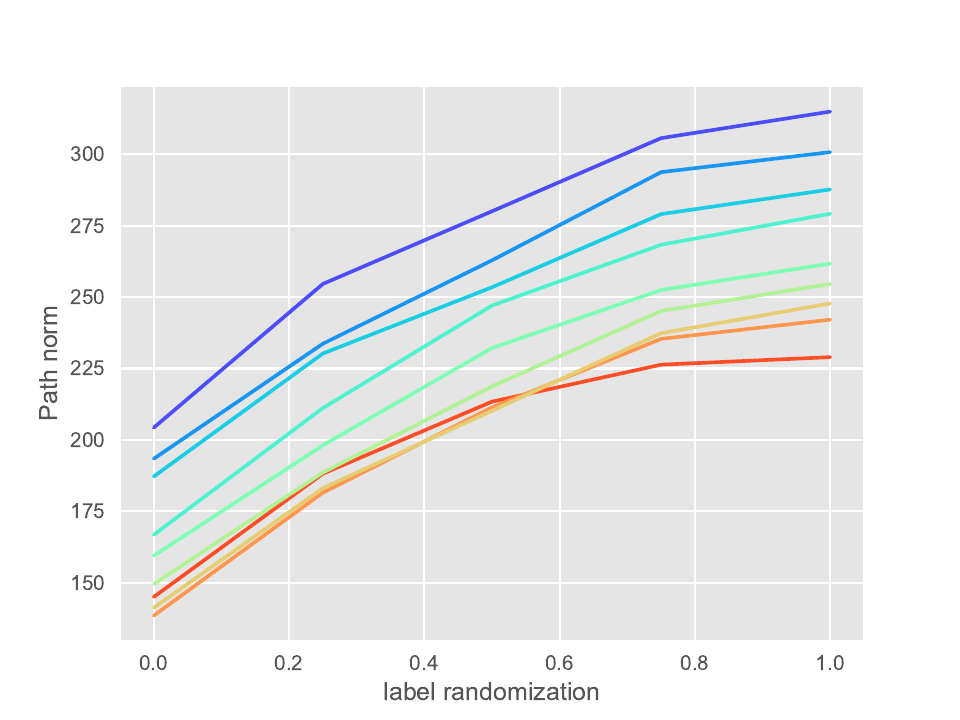}
\hspace{-0.5cm}
\includegraphics[width=0.5\linewidth]{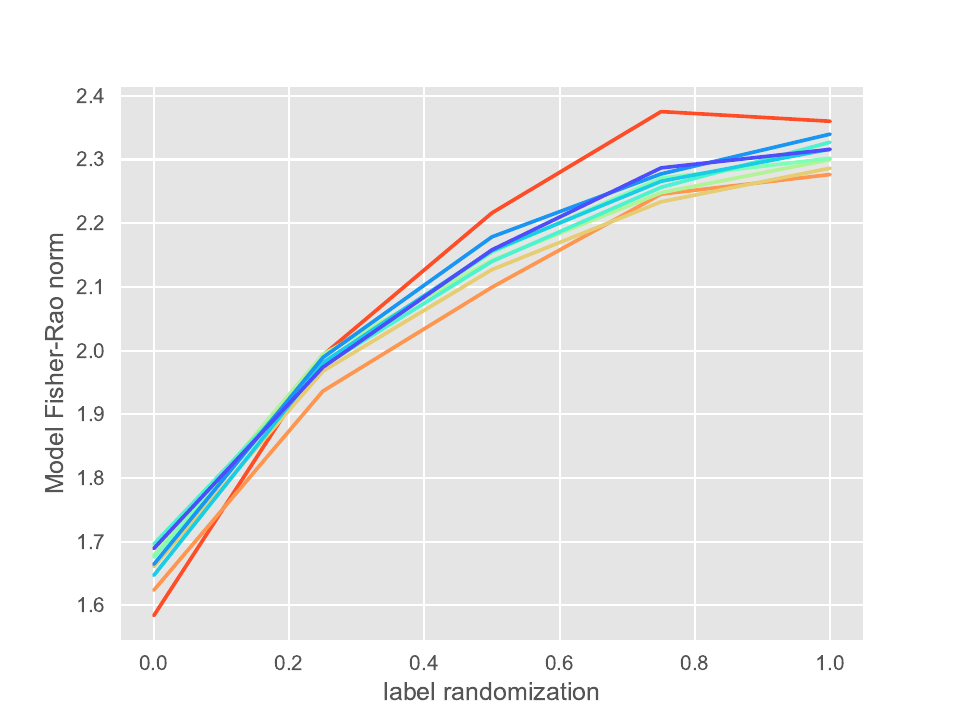}
\hspace{-0.5cm}
\includegraphics[width=0.5\linewidth]{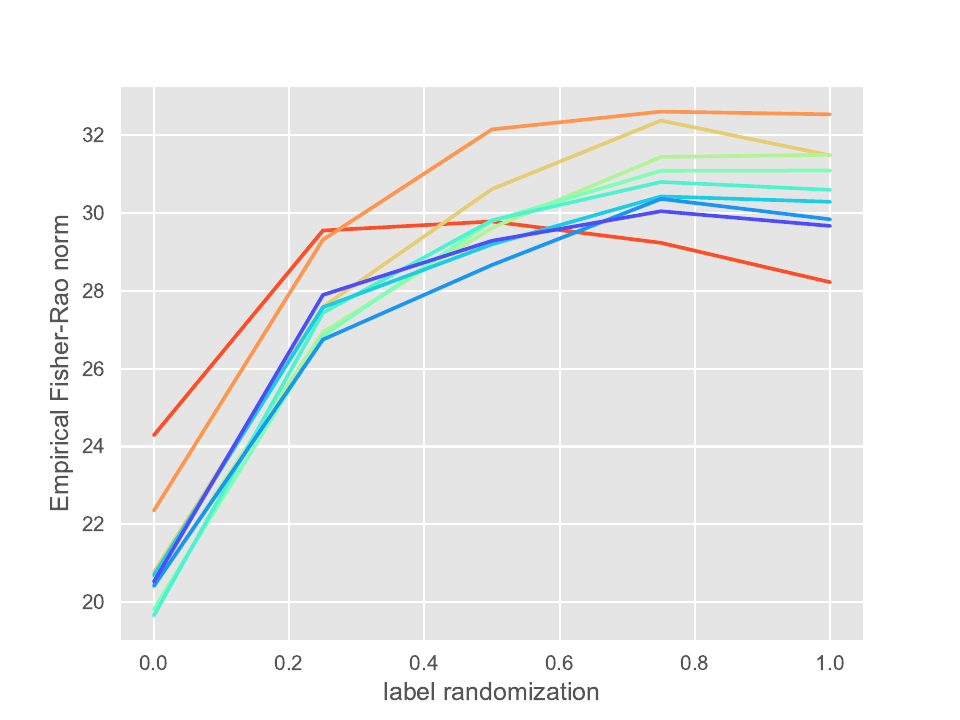}
\hspace{-0.5cm}
\includegraphics[width=0.5\linewidth]{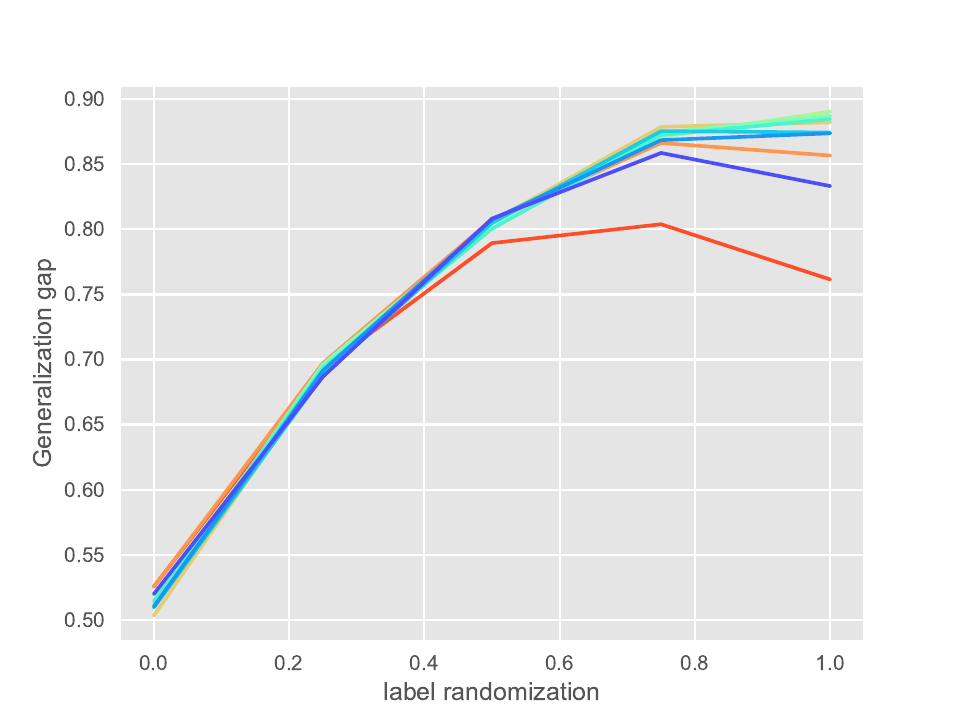}
\caption{Dependence of capacity measures on label randomization after optimizing with gradient descent. The colors show the effect of varying network width from $k=200$ (red) to $k=1000$ (blue) in increments of 100. \label{fig:momentum}}
\end{figure}

% %%%%%%%%%%%%%%%%%%%%%%%%%%%%%%%%%%%%%%%%%%%%%%%%%%%%%
 \section{Further Discussion}

In this paper we studied the generalization puzzle of deep learning from an invariance viewpoint. The notions of invariance come from several angles: information geometry, non-linear local transformations, functional equivalence, algorithmic invariance under parametrization, ``flat'' minima invariance under linear transformations, among many others. We proposed a new non-convex capacity measure using the Fisher-Rao norm and demonstrated its desirable properties from both from the theoretical and the empirical side.
\newpage

{ \small
\bibliography{ref}

\begin{thebibliography}{10}

\bibitem{amari1998natural}
Shun-Ichi Amari.
\newblock Natural gradient works efficiently in learning.
\newblock {\em Neural computation}, 10(2):251--276, 1998.

\bibitem{anthony1999neural}
Martin Anthony and Peter~L Bartlett.
\newblock {\em Neural network learning: Theoretical foundations}.
\newblock cambridge university press, 1999.

\bibitem{bartlett2017spectrally}
Peter Bartlett, Dylan~J Foster, and Matus Telgarsky.
\newblock Spectrally-normalized margin bounds for neural networks.
\newblock {\em arXiv preprint arXiv:1706.08498}, 2017.

\bibitem{bauer2016uniqueness}
Martin Bauer, Martins Bruveris, and Peter~W Michor.
\newblock Uniqueness of the fisher--rao metric on the space of smooth
  densities.
\newblock {\em Bulletin of the London Mathematical Society}, 48(3):499--506,
  2016.

\bibitem{dinh2017sharp}
Laurent Dinh, Razvan Pascanu, Samy Bengio, and Yoshua Bengio.
\newblock Sharp minima can generalize for deep nets.
\newblock {\em arXiv preprint arXiv:1703.04933}, 2017.

\bibitem{hochreiter1997flat}
Sepp Hochreiter and J{\"u}rgen Schmidhuber.
\newblock Flat minima.
\newblock {\em Neural Computation}, 9(1):1--42, 1997.

\bibitem{ferencblog}
Ferenc Husz\'{a}r.
\newblock Blog post on ``(liang et al., 2017), generalization and the
  fisher-rao norm", 2018.

\bibitem{karakida2018universal}
Ryo Karakida, Shotaro Akaho, and Shun-ichi Amari.
\newblock Universal statistics of fisher information in deep neural networks:
  Mean field approach.
\newblock {\em arXiv preprint arXiv:1806.01316}, 2018.

\bibitem{koltchinskii2002empirical}
Vladimir Koltchinskii and Dmitry Panchenko.
\newblock Empirical margin distributions and bounding the generalization error
  of combined classifiers.
\newblock {\em Annals of Statistics}, pages 1--50, 2002.

\bibitem{krizhevsky2012imagenet}
Alex Krizhevsky, Ilya Sutskever, and Geoffrey~E Hinton.
\newblock Imagenet classification with deep convolutional neural networks.
\newblock In {\em Advances in neural information processing systems}, pages
  1097--1105, 2012.

\bibitem{krogh1992simple}
Anders Krogh and John~A Hertz.
\newblock A simple weight decay can improve generalization.
\newblock In {\em Advances in neural information processing systems}, pages
  950--957, 1992.

\bibitem{martens2015optimizing}
James Martens and Roger Grosse.
\newblock Optimizing neural networks with kronecker-factored approximate
  curvature.
\newblock In {\em International Conference on Machine Learning}, pages
  2408--2417, 2015.

\bibitem{neyshabur2017exploring}
Behnam Neyshabur, Srinadh Bhojanapalli, David McAllester, and Nathan Srebro.
\newblock Exploring generalization in deep learning.
\newblock {\em arXiv preprint arXiv:1706.08947}, 2017.

\bibitem{neyshabur2015path}
Behnam Neyshabur, Ruslan~R Salakhutdinov, and Nati Srebro.
\newblock Path-sgd: Path-normalized optimization in deep neural networks.
\newblock In {\em Advances in Neural Information Processing Systems}, pages
  2422--2430, 2015.

\bibitem{neyshabur2015norm}
Behnam Neyshabur, Ryota Tomioka, and Nathan Srebro.
\newblock Norm-based capacity control in neural networks.
\newblock In {\em Conference on Learning Theory}, pages 1376--1401, 2015.

\bibitem{shalev2017failures}
Shai Shalev-Shwartz, Ohad Shamir, and Shaked Shammah.
\newblock Failures of deep learning.
\newblock {\em arXiv preprint arXiv:1703.07950}, 2017.

\bibitem{zhang2016understanding}
Chiyuan Zhang, Samy Bengio, Moritz Hardt, Benjamin Recht, and Oriol Vinyals.
\newblock Understanding deep learning requires rethinking generalization.
\newblock {\em arXiv preprint arXiv:1611.03530}, 2016.

\end{thebibliography}
\bibliographystyle{plain}
}

\newpage
\appendix

\onecolumn

\section{Proofs}
\label{sec:proofs}

\begin{proof}[Proof of Lemma~\ref{lem:induction}]
	Recall the property of the activation function $\sigma(z) = \sigma'(z) z$. Let us prove  for any $0 \leq t \leq s \leq L $, and any $l \in [k_{s+1}]$
	\begin{align}
		\sum_{i \in [k_t], j \in [k_{t+1}]} \frac{\partial O_l^{s+1}}{\partial W^{t}_{ij}} W^{t}_{ij} = O_l^{s+1}(x).
	\end{align}
	
	We prove this statement via induction on the non-negative gap $s-t$. 
	Starting with $s - t = 0$, we have
	\begin{align*}
		& \frac{\partial O^{t+1}_{l}}{\partial W_{il}^{t}} = \frac{\partial O^{t+1}_{l}}{\partial N^{t+1}_{l}}  \frac{\partial N^{t+1}_{l}}{\partial W^{t}_{il}} = \sigma'(N_l^{t+1}(x)) O^{t}_i(x),\\
		& \frac{\partial O^{t+1}_{l}}{\partial W_{ij}^{t}} = 0, \quad \text{for}~j \neq l,
	\end{align*}
	and, therefore,
	\begin{align}
		\label{eq:base-induction}
	\sum_{i \in [k_t], j \in [k_{t+1}]} \frac{\partial O_l^{t+1}}{\partial W^{t}_{ij}} W^{t}_{ij} &=  \sum_{i \in [k_t]} \sigma'(N_l^{t+1}(x)) O^{t}_i(x) W^{t}_{il} = \sigma'(N_l^{t+1}(x)) N_l^{t+1}(x) = O_l^{t+1}(x).
	\end{align}
	This solves the base case when $s - t = 0$.

	Let us assume for general $s-t \leq h$ the induction hypothesis ($h \geq 0$), and let us prove it for $s - t = h+1$. Due to chain-rule in the back-propagation updates
	\begin{align}
		\frac{\partial O_l^{s+1}}{\partial W^{t}_{ij}} = \frac{\partial O_l^{s+1}}{\partial N_l^{s+1}} \sum_{k \in [k_s]} \frac{\partial N_l^{s+1}}{\partial O^{s}_{k}} \frac{\partial O^{s}_{k}}{\partial W^{t}_{ij}}.
	\end{align}
	Using the induction on $\frac{\partial O^{s}_{k}}{\partial W^{t}_{ij}}$ as $(s-1) - t = h$
	\begin{align}
		\sum_{i \in [k_t], j \in [k_{t+1}]} \frac{\partial O^{s}_{k}}{\partial W^{t}_{ij}} W^{t}_{ij} = O_k^s(x),
	\end{align}
	and, therefore,
	\begin{align*}
		&\sum_{i \in [k_t], j \in [k_{t+1}]} \frac{\partial O_l^{s+1}}{\partial W^{t}_{ij}} W^{t}_{ij}\\
		&= \sum_{i \in [k_t], j \in [k_{t+1}]} \frac{\partial O_l^{s+1}}{\partial N_l^{s+1}} \sum_{k \in [k_s]} \frac{\partial N_l^{s+1}}{\partial O^{s}_{k}} \frac{\partial O^{s}_{k}}{\partial W^{t}_{ij}} W^{t}_{ij} \\
		&= \frac{\partial O_l^{s+1}}{\partial N_l^{s+1}} \sum_{k \in [k_s]} \frac{\partial N_l^{s+1}}{\partial O^{s}_{k}} \sum_{i \in [k_t], j \in [k_{t+1}]} \frac{\partial O^{s}_{k}}{\partial W^{t}_{ij}} W^{t}_{ij} \\
		& = \sigma'(N_l^{s+1}(x)) \sum_{k \in [k_s]}  W^s_{kl} O^{s}_k(x) = O_l^{s+1}(x).
	\end{align*}
	This completes the induction argument. 
	In other words, we have proved 
	for any $t, s$ that $t \leq s$, and $l$ is any hidden unit in layer $s$
	\begin{align}
		\sum_{i,j \in {\rm dim}(W^t)} \frac{\partial O^{s+1}_l}{\partial W^{t}_{ij}} W^{t}_{ij} = O^{s+1}_l(x).
	\end{align}
	
	Remark that in the case when there are hard-coded zero weights, the proof still goes through exactly. The reason is, for the base case $s = t$,
	\begin{align*}
		\sum_{i \in [k_t], j \in [k_{t+1}]} \frac{\partial O_l^{t+1}}{\partial W^{t}_{ij}} W^{t}_{ij} &=  \sum_{i \in [k_t]} \sigma'(N_l^{t+1}(x)) O^{t}_i(x) W^{t}_{il} \mathbf{1}(W^{t}_{il} \neq 0)  = \sigma'(N_l^{t+1}(x)) N_l^{t+1}(x) = O_l^{t+1}(x).
	\end{align*}
	and for the induction step, 
	\begin{align*}
		\sum_{i \in [k_t], j \in [k_{t+1}]} \frac{\partial O_l^{s+1}}{\partial W^{t}_{ij}} W^{t}_{ij}
		 = \sigma'(N_l^{s+1}(x)) \sum_{k \in [k_s]}  W^s_{kl} O^{s}_k(x) \mathbf{1}(W^s_{kl} \neq 0) = O_l^{s+1}(x).
	\end{align*}
\end{proof}

\begin{proof}[Proof of Corollary~\ref{cor:largemargin}]
	Observe that $\partial \ell(f, Y) / \partial f =  -y$ if $yf<1$, and $\partial \ell(f, Y) / \partial f = 0$ if $yf \geq 1$. When the output layer has only one unit, we find
	\begin{align*}
		\langle \nabla_\theta \widehat{L}(\theta), \theta \rangle &= (L+1) \widehat{\E}  \left[ \frac{\partial \ell(f_\theta(X), Y)}{\partial f_\theta(X) } f_{\theta}(X) \right]
		=(L+1) \widehat{\E}  \left[ -Yf_{\theta}(X) \mathbf{1}_{Yf_{\theta}(X) < 1}\right] \enspace .
	\end{align*}
	For a stationary point $\theta$, we have $\nabla_\theta \widehat{L}(\theta) = \mathbf{0}$, which implies the LHS of the above equation is 0. Now recall that the second condition that $\theta$ separates the data
	implies implies $-Yf_{\theta}(X) <0$ for any point in the data set. In this case, the RHS equals zero if and only if $Yf_{\theta}(X) \geq 1$. 
\end{proof}

\begin{proof}[Proof of Corollary~\ref{cor:stationary}]
	The proof follows from applying Lemma~\ref{lem:induction}
	\begin{align*}
		0 = \theta^T \nabla_{\theta}\widehat{L}(\theta) = (L+1) \widehat{\E}  \left[ (Y - X^T \prod_{t=0}^{L} W^t )X^T \prod_{t=0}^{L} W^t  \right] \enspace , 
	\end{align*}
	which means
	$\langle w(\theta), \mathbf{X}^T \mathbf{X} w(\theta) -  \mathbf{X}^T \mathbf{Y}  \rangle = 0 $.
\end{proof}

\begin{proof}[Proof of Theorem~\ref{thm:normcomp} (spectral norm)]
	The proof follows from a peeling argument from the right hand side. Recall that $O^t \in \mathbb{R}^{1 \times k_t}$, $W^L \in \mathbb{R}^{k_L \times 1}$ and $|O^L W^L| \leq \Vert W^L \|_{\sigma} \| O^L  \|_{2}$ so one has
	\begin{align*}
		\frac{1}{(L+1)^2}\| \theta \|_{\rm fr}^2 &= \mathbb{E} \left[ |   O^L W^L D^{L+1}  |^2 \right] \\
		&\leq \mathbb{E} \left[ \| W^L \|_{\sigma}^2 \cdot \| O^L  \|_{2}^2 \cdot  |D^{L+1}(X)|^2  \right]  \\
		&= \mathbb{E} \left[ |D^{L+1}(X)|^2 \cdot \| W^L \|_{\sigma}^2 \cdot \| O^{L-1} W^{L-1} D^{L}  \|_{2}^2 \right]  \\ 
		&\leq \mathbb{E} \left[ |D^{L+1}(X)|^2 \cdot \| W^L \|_{\sigma}^2 \cdot \| O^{L-1} W^{L-1} \|_2^2  \cdot \| D^{L}  \|_{\sigma}^2 \right]  \\
		&\leq \mathbb{E} \left[ \| D^{L}  \|_{\sigma}^2 |D^{L+1}(X)|^2 \cdot \| W^L \|_{\sigma}^2 \| W^{L-1} \|_{\sigma}^2 \cdot \| O^{L-1} \|_2^2   \right] \\
		& \leq \mathbb{E} \left[ \| D^{L}  \|_{\sigma}^2 \|D^{L+1}(X)\|_{\sigma}^2  \| O^{L-1} \|_2^2   \right] \cdot  \| W^{L-1} \|_{\sigma}^2 \| W^L \|_{\sigma}^2 \\
		& \ldots \quad \text{repeat the process to bound $\| O^{L-1} \|_2$}\\
		&\leq  \mathbb{E}\left( \| X \|^2  \prod_{t=1}^{L+1} \| D^{t}(X)\|_{\sigma}^2 \right)  \prod_{t=0}^{L} \| W^t \|_{\sigma}^2 = \vvvert \theta \vvvert_{\sigma}^2.
	\end{align*}	
\end{proof}

\begin{proof}[Proof of Theorem~\ref{thm:normcomp} (group norm)]
	The proof still follows a peeling argument from the right. We have
	\begin{align*}
		\frac{1}{(L+1)^2} \| \theta \|_{\rm fr}^2 &= \mathbb{E} \left[ |   O^L W^L D^{L+1}  |^2 \right] \\
		&\leq \mathbb{E} \left[ \| W^L \|_{p,q}^2 \cdot \| O^L  \|_{p^*}^2 \cdot  |D^{L+1}(X)|^2  \right] \quad \text{use \eqref{eq:holder}} \\
		&= \mathbb{E} \left[ |D^{L+1}(X)|^2 \cdot \| W^L \|_{p,q}^2 \cdot \| O^{L-1} W^{L-1} D^{L}  \|_{p^*}^2 \right]  \\ 
		&\leq \mathbb{E} \left[ |D^{L+1}(X)|^2 \cdot \| W^L \|_{p,q}^2 \cdot \| O^{L-1} W^{L-1} \|_{q}^2  \cdot \| D^{L}  \|_{q\rightarrow p^*}^2 \right] \\
		&\leq \mathbb{E} \left[ \| D^{L}  \|_{q\rightarrow p^*}^2 \|D^{L+1}(X)\|_{p,q}^2 \cdot \| W^L \|_{p,q}^2 \| W^{L-1} \|_{p,q}^2 \cdot \| O^{L-1} \|_{p^*}^2   \right] \quad \text{use \eqref{eq:norm-bd}} \\
		&= \mathbb{E} \left[ \| D^{L}  \|_{q\rightarrow p^*}^2 \|D^{L+1}(X)\|_{p,q}^2 \cdot \| O^{L-1} \|_{p^*}^2   \right] \cdot  \| W^{L-1} \|_{p,q}^2 \| W^L \|_{p,q}^2 \\
		& \leq \ldots \quad \text{repeat the process to bound $\| O^{L-1} \|_{p^*}$} \\
		&\leq  \mathbb{E} \left( \| X \|_{p^*}^2  \prod_{t=1}^{L+1} \| D^{t}(X)\|_{q \rightarrow p^*}^2 \right)  \prod_{t=0}^{L} \| W^t \|_{p,q}^2 = \vvvert \theta \vvvert_{p, q}^2
	\end{align*}
In the proof of the first inequality we used Holder's inequality
\begin{align}
	\label{eq:holder}
	\langle w, v \rangle \leq \| w \|_{p} \|v \|_{p^*}
\end{align}
where $\frac{1}{p}+ \frac{1}{p^*} = 1$.
Let's prove
for $v \in \mathbb{R}^n$, $M \in \mathbb{R}^{n \times m}$, we have
\begin{align}
	\| v^T M \|_{q} \leq \| v \|_{p^*} \| M \|_{p, q}.
\end{align}  
Denote each column of $M$ as $M_{\cdot j}$, for $1\leq j \leq m$, 
\begin{align}
	\label{eq:norm-bd}
	\| v^T M \|_{q} = \left( \sum_{j=1}^m |v^T M_{\cdot j}|^q \right)^{1/q} \leq \left( \sum_{j=1}^m \| v \|_{p^*}^q \| M_{\cdot j} \|_p^q \right)^{1/q} = \| v \|_{p^*} \| M \|_{p, q}.
\end{align}
\end{proof}

\begin{proof}[Proof of Theorem~\ref{thm:normcomp} (path norm)]
	The proof is due to Holder's inequality. For any $x \in \mathbb{R}^p$
	\begin{align*}
		&\left| \sum_{i_0, i_1, \ldots, i_L} x_{i_0} W_{i_0 i_1}^0 D^1_{i_1}(x) W_{i_1 i_2}^1 \cdots D^{L}_{i_L}(x) W_{i_L}^L D^{L+1}(x) \right|\\
		&\leq \left( \sum_{i_0, i_1, \ldots, i_L} |x_{i_0} D^1_{i_1}(x) \cdots D^{L}_{i_L}(x) D^{L+1}(x)|^{q^*}  \right)^{1/q^*} \cdot \left( \sum_{i_0, i_1, \ldots, i_L} |W_{i_0 i_1}^0 W_{i_1 i_2}^1 W_{i_2 i_3}^2 \cdots W_{i_L}^L |^q \right)^{1/q}.
	\end{align*}
	Therefore we have
	\begin{align*}
		& \frac{1}{(L+1)^2}\|\theta \|_{\rm fr}^2 = \mathbb{E} \left| \sum_{i_0, i_1, \ldots, i_L} X_{i_0} W_{i_0 i_1}^0 D^1_{i_1}(X) W_{i_1 i_2}^1 \cdots W_{i_L}^L D^{L+1}_{i_L}(X) \right|^2 \\
		& \leq \left( \sum_{i_0, i_1, \ldots, i_L} |W_{i_0 i_1}^0 W_{i_1 i_2}^1 W_{i_2 i_3}^2 \cdots W_{i_L}^L |^q \right)^{2/q} \cdot  \mathbb{E} \left( \sum_{i_0, i_1, \ldots, i_L} |X_{i_0} D^1_{i_1}(X) \cdots D^{L}_{i_L}(X) D^{L+1}(X) |^{q^*} \right)^{2/q^*},
	\end{align*}
	which gives
	\begin{align*}
		\frac{1}{L+1} \| \theta \|_{\rm fr} \leq \left[ \E \left( \sum_{i_0, i_1, \ldots, i_L} |X_{i_0} \prod_{t=1}^{L+1} D^t_{i_t}(X)|^{q^*} \right)^{2/q^*} \right]^{1/2} \cdot \left(\sum_{i_0, i_1, \ldots, i_L} \prod_{t=0}^L |W_{i_t i_{t+1}}^t|^q \right)^{1/q} = \vvvert \pi(\theta) \vvvert_{q}.
	\end{align*}
\end{proof}

\begin{proof}[Proof of Theorem~\ref{thm:normcomp} (matrix-induced norm)]
	The proof follows from the recursive use of the inequality,
	$$
	\| M \|_{p\rightarrow q} \| v \|_p \geq \|v^T M\|_q.
	$$
	We have
	\begin{align*}
		\| \theta \|_{\rm fr}^2 &= \mathbb{E} \left[ |   O^L W^L D^{L+1}  |^2 \right] \\
		&\leq \mathbb{E} \left[ \| W^L \|_{p\rightarrow q}^2 \cdot \| O^L  \|_{p}^2 \cdot  |D^{L+1}(X)|^2  \right]  \\
		&\leq \mathbb{E} \left[ |D^{L+1}(X)|^2 \cdot \| W^L \|_{p\rightarrow q}^2 \cdot \| O^{L-1} W^{L-1} D^{L}  \|_{p}^2 \right]  \\ 
		&\leq \mathbb{E} \left[ |D^{L+1}(X)|^2 \cdot \| W^L \|_{p\rightarrow q}^2 \cdot \| O^{L-1} W^{L-1} \|_{q}^2  \cdot \| D^{L}  \|_{q\rightarrow p}^2 \right] \\
		&\leq \mathbb{E} \left[ \| D^{L}  \|_{q\rightarrow p}^2 \|D^{L+1}(X)\|_{q\rightarrow p}^2 \cdot \| W^L \|_{p \rightarrow q}^2 \| W^{L-1} \|_{p\rightarrow q}^2 \cdot \| O^{L-1} \|_{p}^2   \right] \\
		& \leq \ldots \quad \text{repeat the process to bound $\| O^{L-1} \|_{p}$} \\
		&\leq  \E\left( \| X \|_p^2  \prod_{t=1}^{L+1} \| D^{t}(X)\|_{q \rightarrow p}^2 \right)  \prod_{t=0}^{L} \| W^t \|_{p \rightarrow q}^2 = \vvvert \theta \vvvert_{p \rightarrow q}^2,
	\end{align*}
	where third to last line is because $D^{L+1}(X) \in \mathbb{R}^1$, $|D^{L+1}(X)| = \| D^{L+1}(X)\|_{q \rightarrow p}$.
\end{proof}

\begin{proof}[Proof of Theorem~\ref{thm:normcomp} (chain of induced norm)]
	The proof follows from a different strategy of peeling the terms from the right hand side, as follows,
	\begin{align*}
		\| \theta \|_{\rm fr}^2 &= \mathbb{E} \left[ |   O^L W^L D^{L+1}  |^2 \right] \\
		&\leq \mathbb{E} \left[ \| W^L \|_{p_L \rightarrow p_{L+1}}^2 \cdot \| O^L  \|_{p_L}^2 \cdot  |D^{L+1}(X)|^2  \right]  \\
		&\leq \mathbb{E} \left[ |D^{L+1}(X)|^2 \cdot \| W^L \|_{p_L \rightarrow p_{L+1}}^2 \cdot \| O^{L-1} W^{L-1} D^{L}  \|_{p_L}^2 \right]  \\ 
		&\leq \mathbb{E} \left[ |D^{L+1}(X)|^2 \cdot \| W^L \|_{p_L \rightarrow p_{L+1}}^2 \cdot \| O^{L-1} W^{L-1}\|_{p_{L}}  \| D^{L}  \|_{p_L \rightarrow p_L}^2\right] \\
		&\leq \mathbb{E} \left[ \| D^{L}  \|_{p_L \rightarrow p_L}^2 |D^{L+1}(X)|^2 \cdot \| W^L \|_{p_L \rightarrow p_{L+1}}^2 \| W^{L-1} \|_{p_{L-1}\rightarrow p_{L}}^2 \cdot \| O^{L-1} \|_{p_{L-1}}^2   \right] \\
		&\leq  \mathbb{E}\left( \| X \|_{p_0}^2  \prod_{t=1}^{L+1} \| D^{t}(X)\|_{p_t \rightarrow p_t}^2 \right)  \prod_{t=0}^{L} \| W^t \|_{p_t \rightarrow p_{t+1}}^2 = \vvvert \theta \vvvert_{P}^2.
	\end{align*}	
\end{proof}

\begin{proof}[Proof of Lemma~\ref{lem:star-shape}]
	\begin{align*}
		\frac{d}{dr} \|r \theta \|_{\rm fr}^2 & = \mathbb{E}\left[ 2 \langle \theta,  \nabla_\theta f_{r\theta}(X) \rangle f_{r\theta}(X) \right] \\
		& = \mathbb{E}\left[ \frac{2(L+1)}{r} f_{r\theta}(X)  f_{r\theta}(X) \right]  \quad \text{use Lemma~\ref{lem:induction}} \\
		& = \frac{2(L+1)}{r} \|r \theta \|_{\rm fr}^2
	\end{align*}
	The last claim can be proved through solving the simple ODE. 
\end{proof}

\begin{proof}[Proof of Lemma~\ref{lem:convex.f}]
	Let us first construct $\theta' \in \Theta_{L+1}$ that realizes $\lambda f_{\theta_1} + (1-\lambda) f_{\theta_2}$. The idea is very simple:
	we put $\theta_1$ and $\theta_2$ networks side-by-side, then construct an additional output layer with weights $\lambda$, $1-\lambda$ on the output of $f_{\theta_1}$ and $f_{\theta_2}$, and the final output layer is passed through $\sigma(x) = x$. One can easily see that our key Lemma~\ref{lem:induction} still holds for this network: the interaction weights between $f_{\theta_1}$ and $f_{\theta_2}$ are always hard-coded as $0$. Therefore we have constructed a $\theta' \in \Theta_{L+1}$ that realizes $\lambda f_{\theta_1} + (1-\lambda) f_{\theta_2}$. 
	
	Now recall that
	\begin{align*}
		\frac{1}{L+2} \| \theta' \|_{\rm fr} &= \left( \E  f_{\theta'}^2  \right)^{1/2} \\
		& =  \left( \E ( \lambda f_{\theta_1} + (1-\lambda) f_{\theta_2} )^2 \right)^{1/2} \\
		& \leq \lambda \left( \E f_{\theta_1}^2 \right)^{1/2} + (1-\lambda) \left( \E f_{\theta_2}^2 \right)^{1/2} \leq 1
	\end{align*}
	because $\E[f_{\theta_1} f_{\theta_2}] \leq \left( \E f_{\theta_1}^2 \right)^{1/2} \left( \E f_{\theta_2}^2 \right)^{1/2}$.
\end{proof}

\begin{proof}[Proof of Theorem~\ref{thm:generalization-linear}]
	Due to Eqn.~\eqref{eq:FR}, one has
	\begin{align*}
		\frac{1}{(L+1)^2} \| \theta \|^2_{\rm fr} &= \mathbb{E} \left[ v(\theta, X)^T XX^T v(\theta, X) \right] \\
		&=  v(\theta)^T \mathbb{E} \left[ XX^T \right] v(\theta)
	\end{align*}
	because in the linear case $v(\theta, X) = W^0 D^1(x) W^1 D^2(x) \cdots D^L(x) W^L D^{L+1}(x) = \prod_{t=0}^L W^t =: v(\theta) \in \mathbb{R}^p$.
	Therefore
	\begin{align*}
		\R_N\left( B_{\rm fr}(\gamma)  \right) & = \E_{\epsilon} \sup_{\theta \in B_{\rm fr}(\gamma) } \frac{1}{N} \sum_{i=1}^N \epsilon_i f_{\theta}(X_i) \\
		& = \E_{\epsilon} \sup_{\theta \in B_{\rm fr}(\gamma) } \frac{1}{N} \sum_{i=1}^N \epsilon_i X_i^T v(\theta) \\
		& = \E_{\epsilon} \sup_{\theta \in B_{\rm fr}(\gamma) }  \frac{1}{N} \left\langle \sum_{i=1}^N \epsilon_i X_i, v(\theta) \right\rangle \\
		& \leq \gamma \E_{\epsilon} \frac{1}{N} \left\| \sum_{i=1}^N \epsilon_i X_i \right\|_{\left[ \mathbb{E} (XX^T) \right]^{-1}} \\
		& \leq \gamma \frac{1}{\sqrt{N}} \sqrt{ \frac{1}{N} \E_{\epsilon} \left\| \sum_{i=1}^N \epsilon_i X_i \right\|_{\left[ \mathbb{E} (XX^T) \right]^{-1}}^2 } \\
		& =  \gamma \frac{1}{\sqrt{N}} \sqrt{ \left\langle \frac{1}{N} \sum_{i=1}^N X_i X_i^T, \left[ \mathbb{E} (XX^T) \right]^{-1} \right\rangle}.
	\end{align*}
	Therefore
	\begin{align*}
		\E \R_N\left( B_{\rm fr}(\gamma)  \right) \leq \gamma \frac{1}{\sqrt{N}} \sqrt{  \E \left\langle \frac{1}{N} \sum_{i=1}^N X_i X_i^T, \left[ \mathbb{E} (XX^T) \right]^{-1} \right\rangle} = \gamma \sqrt{\frac{p}{N}}.
	\end{align*}
	
\end{proof}

\begin{proof}[Proof of Proposition~\ref{thm:generalization-rectified}]
	If $\mathcal{G} \subseteq \F$ then one has the lower bound $\R_N(\mathcal{G}) \leq \R_N(\F)$ on the empirical Rademacher complexity of $\F$. One can also obtain an upper bound by examining how the sub-space of functions $\mathcal{G}$ approximates $\F$. For each $f \in \F$ consider the closest point $g_f \in \mathcal{G}$ to $f$,
	%Without loss of generality, consider only $\gamma = 1$. 
	%The Rademacher complexity of Fisher-Rao ball can be bounded in the following way, by looking at how sub-space of functions 
	%$\F_{\sigma}$ approximates $\F_{\rm fr}$. Take  $r = 1/ [\widehat{\E}\Vert X \Vert^2]^{1/2}$, and denote $g_f$ 
	\begin{align*}
		g_f :=  \argmin_{g \in \F_{\sigma}} \| f - g \|_\infty \enspace .
	\end{align*}
	Then the empirical Rademacher complexity $\R_N(\F)$ is upper-bounded in terms of $\R_N(\mathcal{G})$ by
	\begin{align*}
		\R_N(\F) & = \E_{\epsilon} \sup_{f \in \F} \frac{1}{N} \sum_{i=1}^N \epsilon_i f(X_i) \enspace , \\
		&\leq \E_{\epsilon} \sup_{f \in \F} \frac{1}{N} \sum_{i=1}^N \epsilon_i [f(X_i) - g_f(X_i)] + \R_N(\mathcal{G}) \enspace , \\
		&\leq \sup_{f \in \F}\inf_{g \in \mathcal{G}} \| f - g \|_\infty + \R_N(\mathcal{G}) \enspace .
	\end{align*}
	Therefore, taking expectation values over the data gives,
	\begin{equation*}
	\E \R_N(\F) \leq \sup_{f \in \F}\inf_{g \in \mathcal{G}} \| f - g \|_\infty + \E \R_N(\mathcal{G}) \enspace .
	\end{equation*}
	Setting $\F = \F_{\rm fr}(1)$ without loss of generality, we obtain the required result by appropriate choice of $\mathcal{G}\subseteq\F_{\rm fr}(1)$. 
	
	Setting $\mathcal{G}=\F_\sigma(r)$ with $r = 1/ [\widehat{\E}\Vert X \Vert^2]^{1/2}$ gives (Remark 4.1, Theorem 1.1 in \cite{bartlett2017spectrally}),
	\begin{equation*}
	\E \R_N(\F_{\rm fr}(1)) \leq \sup_{f \in \F_{\rm fr}(1)}\inf_{g \in \F_\sigma(r)} \| f - g \|_\infty + \frac{{\rm Polylog}}{\sqrt{N}} \enspace .
	\end{equation*}
	Setting $\mathcal{G} = \F_{p, q}(r)$ with $r = 1/(k^{[1/p^{*} - 1/q]_{+}})^L \max_i \| X_i \|_{p^*}$ gives (Remark 4.2, Theorem 1 in \cite{neyshabur2015norm})
	\begin{equation*}
	\E \R_N(\F_{\rm fr}(1)) \leq \sup_{f \in \F_{\rm fr}(1)}\inf_{g \in \F_{p, q}(r)} \| f - g \|_\infty + \frac{2^L {\rm Polylog}}{\sqrt{N}} \enspace .
	\end{equation*}
	Setting $\mathcal{G} = \F_{\pi, 1}(r)$ with $r = 1/\max_i \| X_i \|_\infty$ gives (Remark 4.3, Corollary in \cite{neyshabur2015norm})
	\begin{equation*}
	\E \R_N(\F_{\rm fr}(1)) \leq \sup_{f \in \F_{\rm fr}(1)}\inf_{g \in \F_{\pi, 1}(r)} \| f - g \|_\infty + \frac{2^L {\rm Polylog}}{\sqrt{N}} \enspace .
	\end{equation*}
	In all cases data-dependent pre-factors exactly cancel out and moreover the first term is in function space, not in parameter space.
\end{proof}

\subsection{Invariance of natural gradient}
\label{sec:inv_nat_grad}
Consider the continuous-time analog of natural gradient flow, 
	\begin{align}
		\label{eq:theta}
		d \theta_t = - \mathbf{I}(\theta_t)^{-1} \nabla_{\theta} L(\theta_t) dt,
	\end{align}
	where $\theta \in \mathbb{R}^p$. Consider a differentiable transformation from one parametrization to another $\theta \mapsto \xi \in \mathbb{R}^q$ denoted by $\xi(\theta): \mathbb{R}^p \rightarrow \mathbb{R}^q$. Denote the Jacobian $\mathbf{J}_{\xi}(\theta)  = \frac{\partial (\xi_1, \xi_2, \ldots, \xi_q)}{\partial (\theta_1, \theta_2, \ldots, \theta_p)} \in \mathbb{R}^{q \times p}$. Define the loss function $\tilde{L}: \xi \rightarrow \mathbb{R}$ that satisfies
	\begin{align*}
		L(\theta) = \tilde{L} (\xi (\theta)) =  \tilde{L} \circ \xi(\theta), 
	\end{align*}
	and denote $\tilde{\bf{I}}(\xi)$ as the Fisher Information on $\xi$ associated with $\tilde{L}$. 
	Consider also the natural gradient flow on the $\xi$ parametrization,
	\begin{align}
		\label{eq:xi}
		d \xi_t = - \tilde{\mathbf{I}}(\xi_t)^{-1} \nabla_{\xi} \tilde{L}(\xi_t) dt.
	\end{align}
	Intuitively, one can show that the natural gradient flow is ``invariant'' to the specific parametrization of the problem.
\begin{lem}[Parametrization invariance]
	\label{lem:invariance}
	Denote $\theta \in \mathbb{R}^p$, and the differentiable transformation from one parametrization to another $\theta \mapsto \xi \in \mathbb{R}^q$ as $\xi(\theta): \mathbb{R}^p \rightarrow \mathbb{R}^q$. Assume $\mathbf{I}(\theta)$, $\tilde{\mathbf{I}}(\xi)$ are invertible, and consider two natural gradient flows $\{\theta_t, t>0\}$ and $\{\xi_t, t>0\}$ defined in Eqn.~\eqref{eq:theta} and \eqref{eq:xi} on $\theta$ and $\xi$ respectively. 
	
	(1) Re-parametrization: if $q = p$, and assume $\mathbf{J}_{\xi}(\theta)$ is invertible, then natural gradient flow on the two parameterizations satisfies, 
	$$
	\xi(\theta_t) = \xi_t, \quad \forall t,
	$$
	if the initial locations $\theta_0, \xi_0$ are equivalent in the sense $\xi(\theta_0) = \xi_0$.
	
	(2) Over-parametrization: If $q > p$ and $\xi_t = \xi(\theta_t)$ at some fixed time $t$, then the infinitesimal change satisfies
	$$
	\xi(\theta_{t + dt}) - \xi(\theta_{t}) =  M_t (\xi_{t+dt} - \xi_{t}), \quad \text{$M_t$ has eigenvalues either $0$ or $1$}
	$$ 
	where $M_t = \mathbf{I} (\xi_t)^{-1/2} (I_{q} -  U_{\perp} U_{\perp}^T) \mathbf{I} (\xi_t)^{1/2}$, and $U_\perp$ denotes the null space of $\mathbf{I} (\xi)^{1/2} \mathbf{J}_{\xi}(\theta)$.
\end{lem}
\begin{proof}[Proof of Lemma~\ref{lem:invariance}]
	From basic calculus, one has
	\begin{align*}
		\nabla_{\theta} L(\theta) &=  \mathbf{J}_\xi(\theta)^T \nabla_\xi \tilde{L}(\xi) \\
		\mathbf{I}(\theta) &= \mathbf{J}_\xi(\theta)^T \tilde{\mathbf{I}}(\xi) \mathbf{J}_\xi(\theta)
	\end{align*}
	Therefore, plugging in the above expression into the natural gradient flow in $\theta$
	\begin{align*}
		d \theta_t &= - \mathbf{I}(\theta_t)^{-1} \nabla_{\theta} L(\theta_t) dt \\
		&= - [\mathbf{J}_\xi(\theta_t)^T \tilde{\mathbf{I}}(\xi(\theta_t)) \mathbf{J}_\xi(\theta_t)]^{-1} \mathbf{J}_\xi(\theta_t)^T \nabla_\xi \tilde{L}(\xi(\theta_t)) dt.
	\end{align*}
	
	In the re-parametrization case, $\mathbf{J}_\xi(\theta)$ is invertible, and assuming $\xi_t = \xi(\theta_t)$,
	\begin{align*}
		d \theta_t 
		&= - [\mathbf{J}_\xi(\theta_t)^T \tilde{\mathbf{I}}(\xi(\theta_t)) \mathbf{J}_\xi(\theta_t)]^{-1} \mathbf{J}_\xi(\theta_t)^T \nabla_\xi \tilde{L}(\xi(\theta_t)) dt \\
		&=  - \mathbf{J}_\xi(\theta_t)^{-1} \tilde{\mathbf{I}}(\xi(\theta_t))^{-1} \nabla_\xi \tilde{L}(\xi(\theta_t)) dt \\
		\mathbf{J}_\xi(\theta_t) d \theta_t   &= - \tilde{\mathbf{I}}(\xi(\theta_t))^{-1} \nabla_\xi \tilde{L}(\xi(\theta_t)) dt \\
 d\xi(\theta_t) &=   - \tilde{\mathbf{I}}(\xi(\theta_t))^{-1} \nabla_\xi \tilde{L}(\xi(\theta_t)) dt = - \tilde{\mathbf{I}}(\xi_t)^{-1} \nabla_\xi \tilde{L}(\xi_t) dt.
	\end{align*}
	What we have shown is that under $\xi_t = \xi(\theta_t)$, $\xi(\theta_{t+dt}) = \xi_{t+dt}$. Therefore, if $\xi_0 = \xi(\theta_0)$, we have that $\xi_t = \xi(\theta_t)$. 
	
	In the over-parametrization case, $\mathbf{J}_\xi(\theta) \in \mathbb{R}^{q \times p}$ is a non-square matrix. For simplicity of derivation, abbreviate $B := \mathbf{J}_\xi(\theta) \in \mathbb{R}^{q \times p}$. We have
	\begin{align*}
		d\theta_t = \theta_{t+dt} - \theta_{t} &= - \mathbf{I}(\theta_t)^{-1} \nabla_{\theta} L(\theta_t) dt \\
		& = - [B^T \tilde{\mathbf{I}}(\xi) B]^{-1} B^T \nabla_\xi \tilde{L}(\xi(\theta_t)) dt \\
		B(\theta_{t+dt} - \theta_{t})&= - B \left[ B^T \tilde{\mathbf{I}}(\xi) B  \right]^{-1} B^T \tilde{L}(\xi(\theta_t)) dt.
	\end{align*}
	Via the Sherman-Morrison-Woodbury formula 
	\begin{align*}
		 \left[ I_q + \frac{1}{\epsilon} \tilde{\mathbf{I}} (\xi)^{1/2} B B^T \tilde{\mathbf{I}} (\xi)^{1/2} \right]^{-1}   & = I_q  - \tilde{\mathbf{I}} (\xi)^{1/2} B (\epsilon I_p + B^T \tilde{\mathbf{I}} (\xi) B )^{-1} B^T \tilde{\mathbf{I}} (\xi)^{1/2}
	\end{align*}
	Denoting $\tilde{\mathbf{I}} (\xi)^{1/2} B B^T \tilde{\mathbf{I}} (\xi)^{1/2} = U \Lambda U^T$, we have that ${\rm rank}(\Lambda) \leq p < q$. Therefore, the LHS as 
	\begin{align*}
		 \left[ I_q + \frac{1}{\epsilon} \tilde{\mathbf{I}} (\xi)^{1/2} B B^T \tilde{\mathbf{I}} (\xi)^{1/2} \right]^{-1}  & =  U \left[ I_q + \frac{1}{\epsilon} \Lambda \right]^{-1} U^{T}  \\
		\lim_{\epsilon \rightarrow 0} \left[ I_q + \frac{1}{\epsilon} \tilde{\mathbf{I}} (\xi)^{1/2} B B^T \tilde{\mathbf{I}} (\xi)^{1/2} \right]^{-1}  &= U_{\perp} U_{\perp}^T
	\end{align*}
	where $U_{\perp}$ corresponding to the space associated with zero eigenvalue of $\tilde{\mathbf{I}} (\xi)^{1/2} B B^T \tilde{\mathbf{I}} (\xi)^{1/2}$. 
	Therefore taking $\epsilon \rightarrow 0$, we have
	\begin{align*}
		\lim_{\epsilon \rightarrow 0} \left[ I_q + \frac{1}{\epsilon} \tilde{\mathbf{I}} (\xi)^{1/2} B B^T \tilde{\mathbf{I}} (\xi)^{1/2} \right]^{-1}   & = \lim_{\epsilon \rightarrow 0} I_q- \tilde{\mathbf{I}} (\xi)^{1/2} B (\epsilon I_p + B^T \tilde{\mathbf{I}} (\xi) B )^{-1} B^T \tilde{\mathbf{I}} (\xi)^{1/2}\\
		\tilde{\mathbf{I}} (\xi)^{-1/2} U_{\perp} U_{\perp}^T \tilde{\mathbf{I}} (\xi)^{-1/2} &= \tilde{\mathbf{I}} (\xi)^{-1} - B (B^T \tilde{\mathbf{I}} (\xi) B )^{-1} B^T
	\end{align*}
	where only the last step uses the fact $\tilde{\mathbf{I}} (\xi)$ is invertible.
Therefore
	\begin{align*}
		\xi(\theta_{t+dt}) - \xi(\theta_t) & = 	B(\theta_{t+dt} - \theta_{t}) \\
		&= - B \left[ B^T \mathbf{I}_n(\xi) B  \right]^{-1} B^T \nabla_{\xi} \tilde{L}(\xi) dt \\
		&= - \eta \mathbf{I} (\xi)^{-1/2} (I_{d} -  U_{\perp} U_{\perp}^T) \mathbf{I} (\xi)^{-1/2} \nabla_{\xi} \tilde{L}(\xi) dt \\
		 &= \mathbf{I} (\xi)^{-1/2} (I_{d} -  U_{\perp} U_{\perp}^T) \mathbf{I} (\xi)^{1/2}  \left\{ \mathbf{I} (\xi)^{-1} \nabla_{\xi} \tilde{L}(\xi) dt \right\} \\
		 &= M_t (\xi_{t+dt} - \xi_t).
	\end{align*}
	 The above claim asserts that in the over-parametrized setting, running natural gradient in the over-parametrized space is nearly ``invariant'' in the following sense: if $\xi(\theta_t) = \xi_t$, then
	 \begin{align*}
	 	\xi(\theta_{t+dt}) - \xi(\theta_t) = M_t \left( \xi_{t+dt} - \xi_t \right) \\
		M_t = \mathbf{I} (\xi_t)^{-1/2} (I_{q} -  U_{\perp} U_{\perp}^T) \mathbf{I} (\xi_t)^{1/2}
	 \end{align*}
	 and we know $M_t$ has eigenvalue either $1$ or $0$. In the case when $p = q$ and $\mathbf{J}_{\xi}(\theta)$ has full rank, it holds that $M_t = I$ is the identity matrix, reducing the problem to the re-parametrized case. 
	 
\end{proof}

\section{Experimental details}
In the realistic $K$-class classification context there is no activation function on the $K$-dimensional output layer of the network ($\sigma_{L+1}(x) = x$) and we focus on ReLU activation $\sigma(x) = \max\{0,x\}$ for the intermediate layers. The loss function is taken to be the cross entropy $\ell(y', y) = - \langle e_{y} , \log g(y') \rangle$, where $e_y \in \mathbb{R}^K$ denotes the one-hot-encoded class label and $g(z)$ is the softmax function defined by,
\begin{align*}
	g(z) 
		& = \left(\frac{\exp(z_1)}{\sum_{k=1}^K \exp(z_k)},\ldots,\frac{\exp(z_K)}{\sum_{k=1}^K \exp(z_k)}\right)^{T} \enspace .
\end{align*}
It can be shown that the gradient of the loss function with respect to the output of the neural network is $\nabla \ell(f, y) = - \nabla \langle e_y, \log g(f) \rangle  = g(f) - e_y$, so
plugging into the general expression for the Fisher-Rao norm we obtain,
\begin{align}
	\Vert \theta \Vert_{\rm fr}^2
		& = (L+1)^2 \mathbb{E} [\{\langle g(f_\theta(X)), f_\theta(X) \rangle - f_\theta(X)_Y\}^2].
\end{align}

In practice, since we do not have access to the population density $p(x)$ of the covariates, we estimate the Fisher-Rao norm by sampling from a test set of size $m$, leading to our final formulas
\begin{align}
	\Vert \theta \Vert^2_{\rm fr} 
		& = (L+1)^2\frac{1}{m}\sum_{i=1}^m \sum_{y=1}^K g(f_\theta(x_i))_y [\langle g(f_\theta(x_i)), f_\theta(x_i) \rangle - f_\theta(x_i)_y]^2 \enspace , \\
	\Vert \theta \Vert^2_{\rm fr,emp} 
		& = (L+1)^2\frac{1}{m}\sum_{i=1}^m [\langle g(f_\theta(x_i)), f_\theta(x_i) \rangle - f_\theta(x_i)_{y_i}]^2 \enspace .
\end{align}

\newpage
\subsection{Additional experiments and figures}

%\begin{figure}
%\centering
%\includegraphics[width=0.23\textwidth]{norm-sim-1.pdf}
%\includegraphics[width=0.23\textwidth]{norm-sim-2.pdf}
%\includegraphics[width=0.23\textwidth]{norm-diff-1.pdf}
%\includegraphics[width=0.23\textwidth]{norm-diff-2.pdf}
%\caption{The levelsets of Fisher-Rao norm (solid) and path-$2$ norm (dotted). The color denotes the value of the norm.}
%\label{fig:contour}
%\end{figure}

\begin{figure}[pht]
\centering
\includegraphics[width=0.4\textwidth]{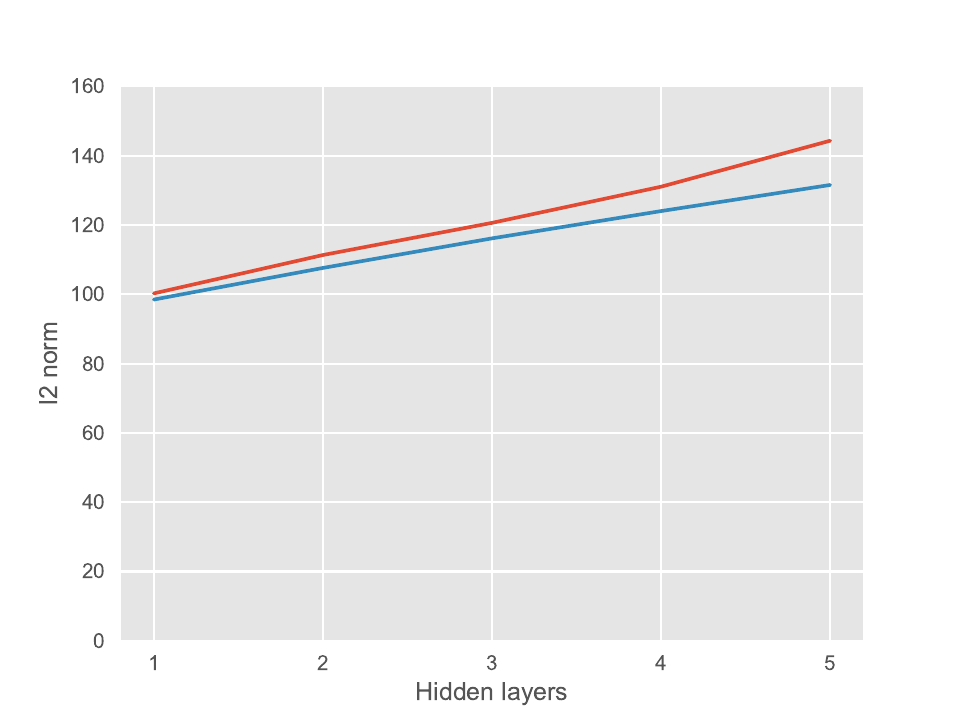}
\includegraphics[width=0.4\textwidth]{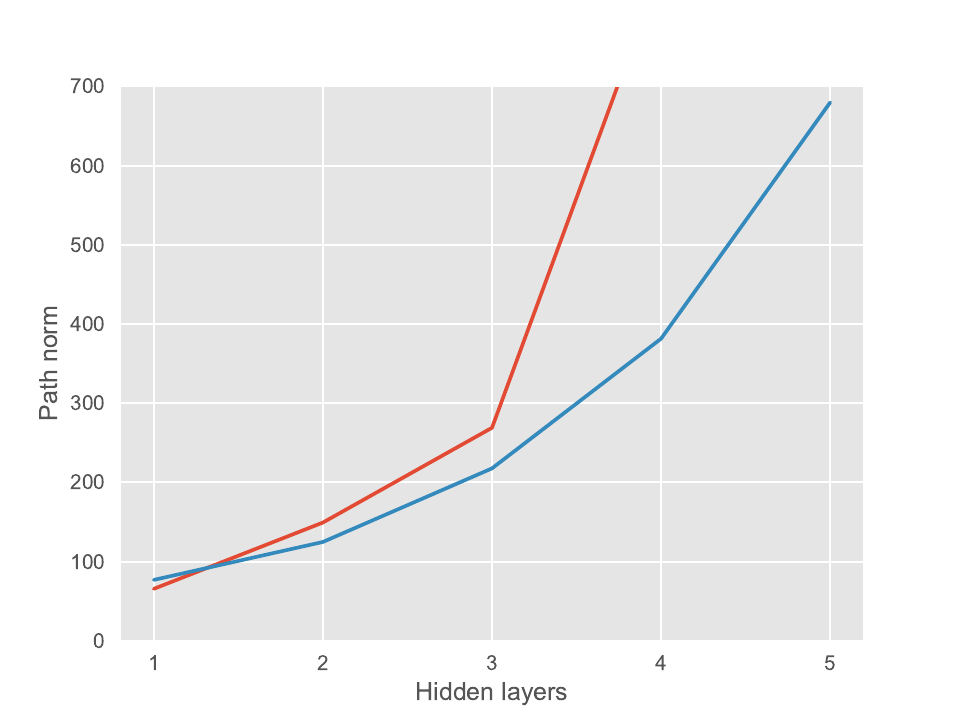}
\includegraphics[width=0.4\textwidth]{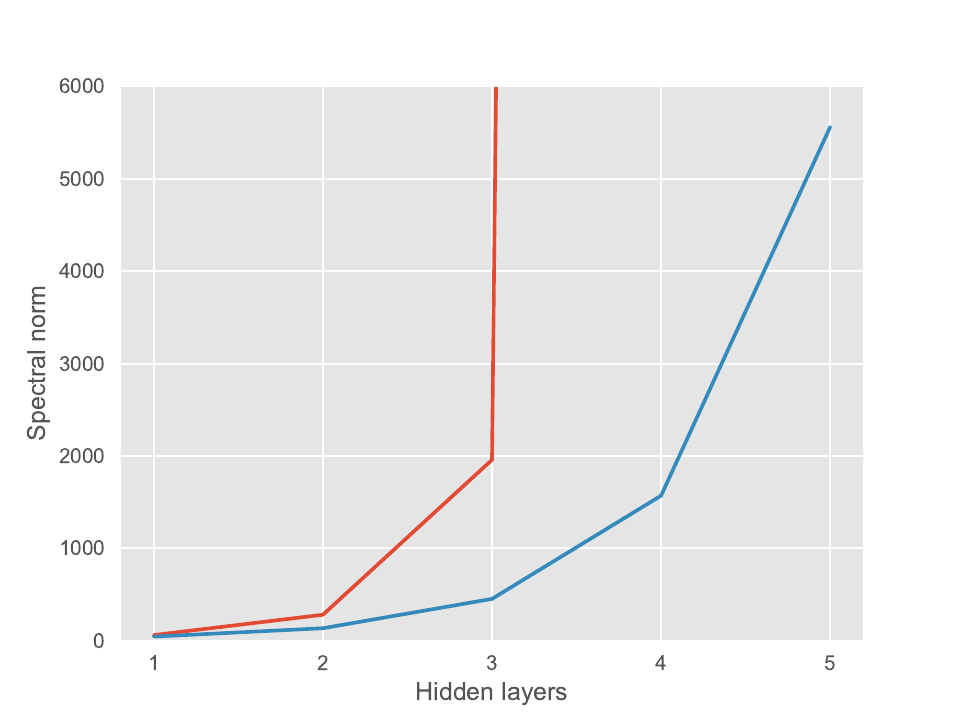}
\includegraphics[width=0.4\textwidth]{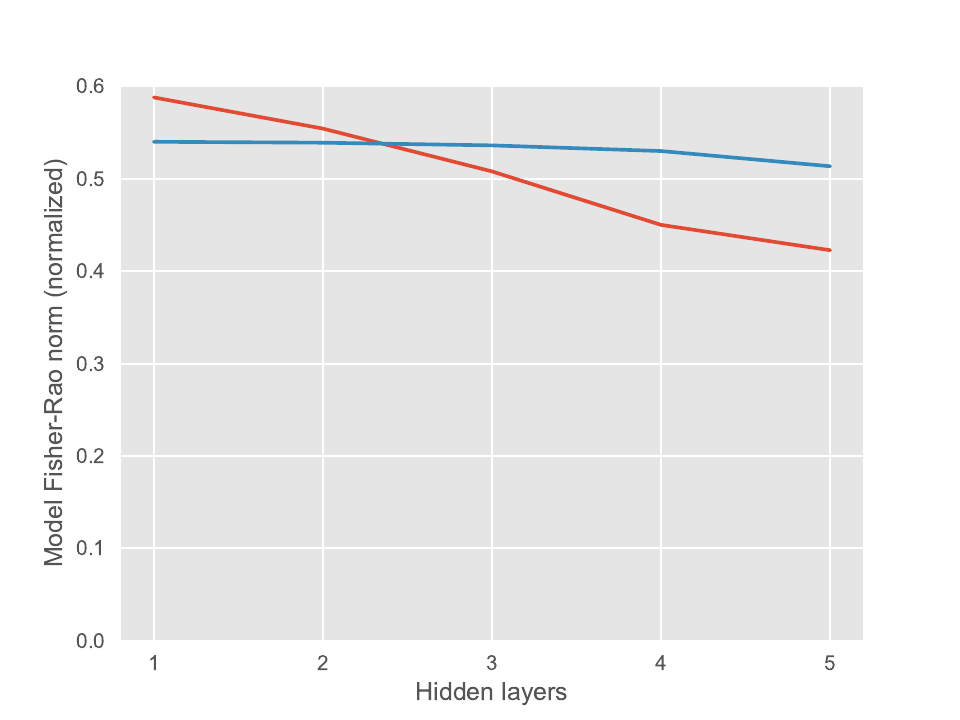}
\includegraphics[width=0.4\textwidth]{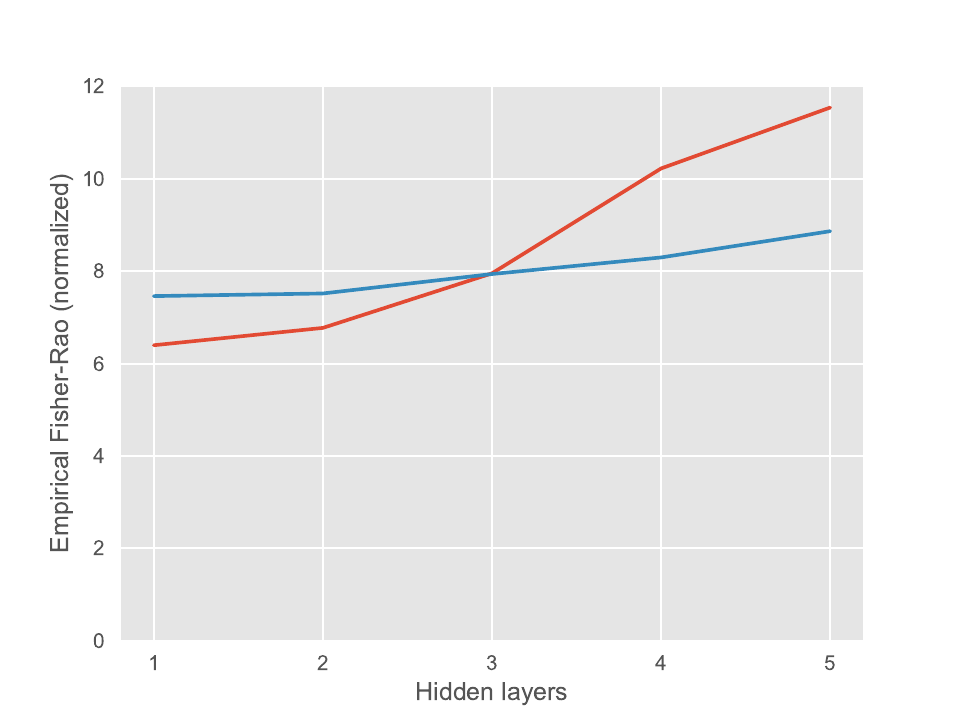}
\caption{Dependence of different norms on depth $L$ ($k = 500$) after optimzing with vanilla gradient descent (red) and natural gradient descent (blue). The Fisher-Rao norms are normalized by $L+1$.}
\label{fig:l}
\end{figure}

\begin{figure}[pht]
\centering
\includegraphics[width=0.4\textwidth]{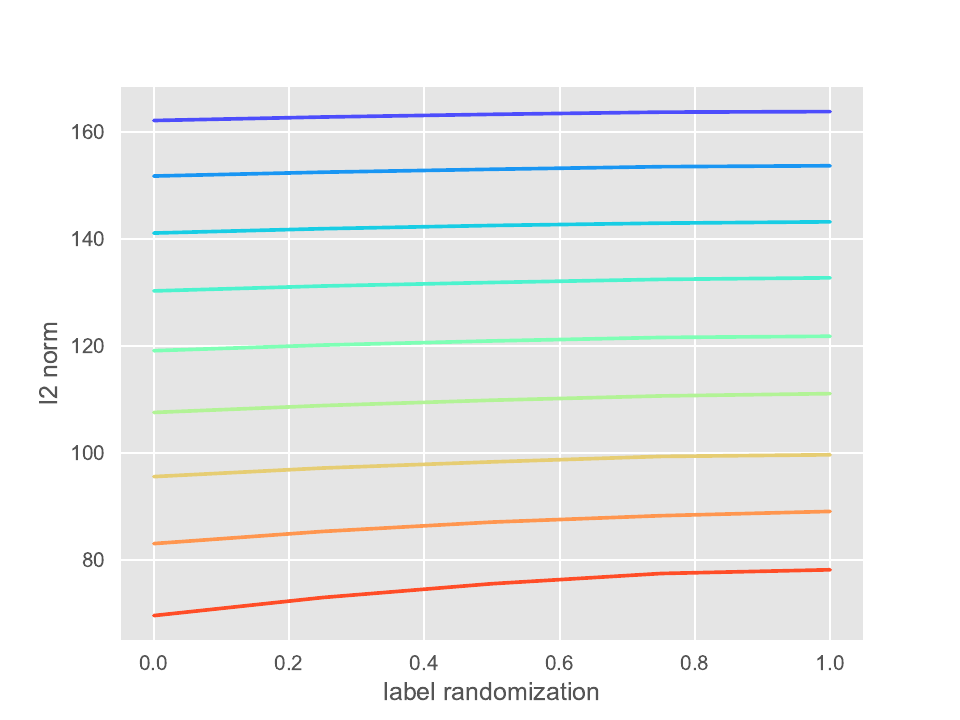}
\includegraphics[width=0.4\textwidth]{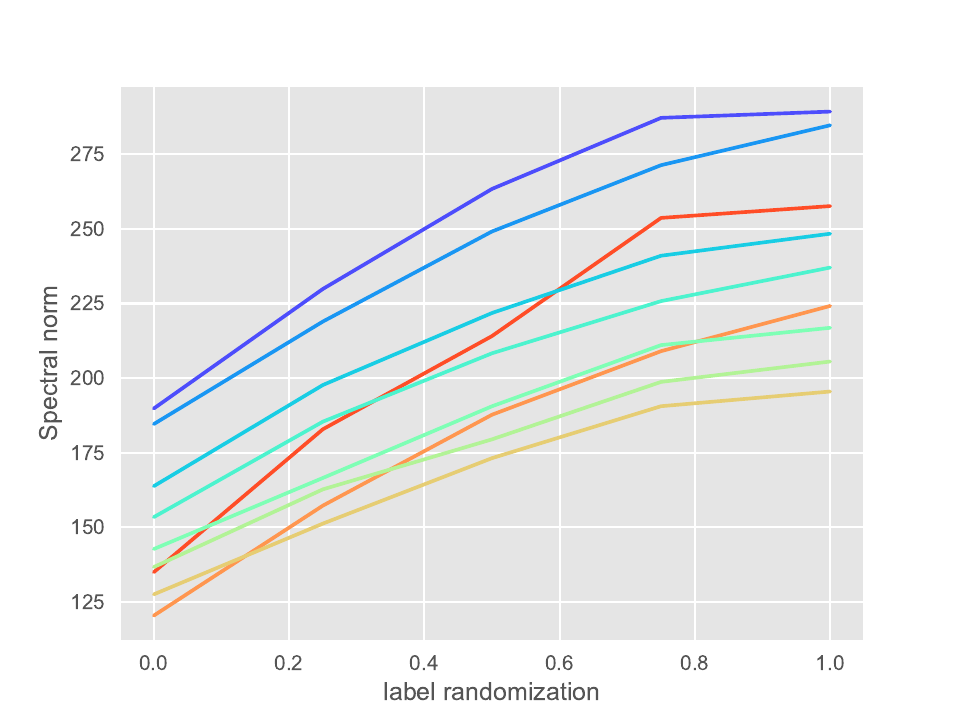}
\includegraphics[width=0.4\textwidth]{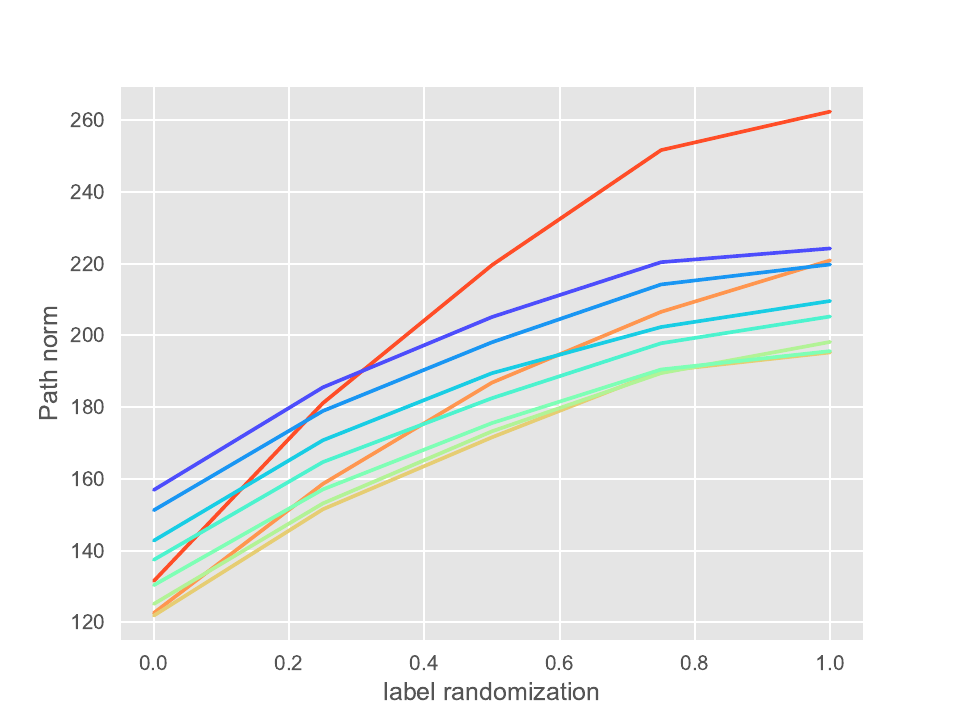}
\includegraphics[width=0.4\textwidth]{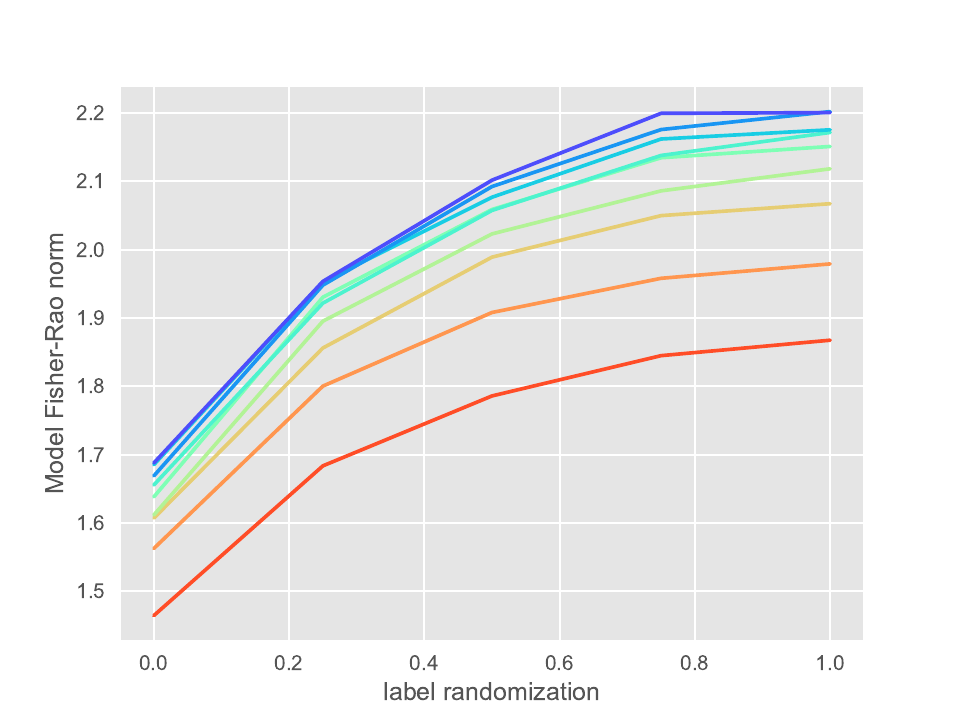}
\includegraphics[width=0.4\textwidth]{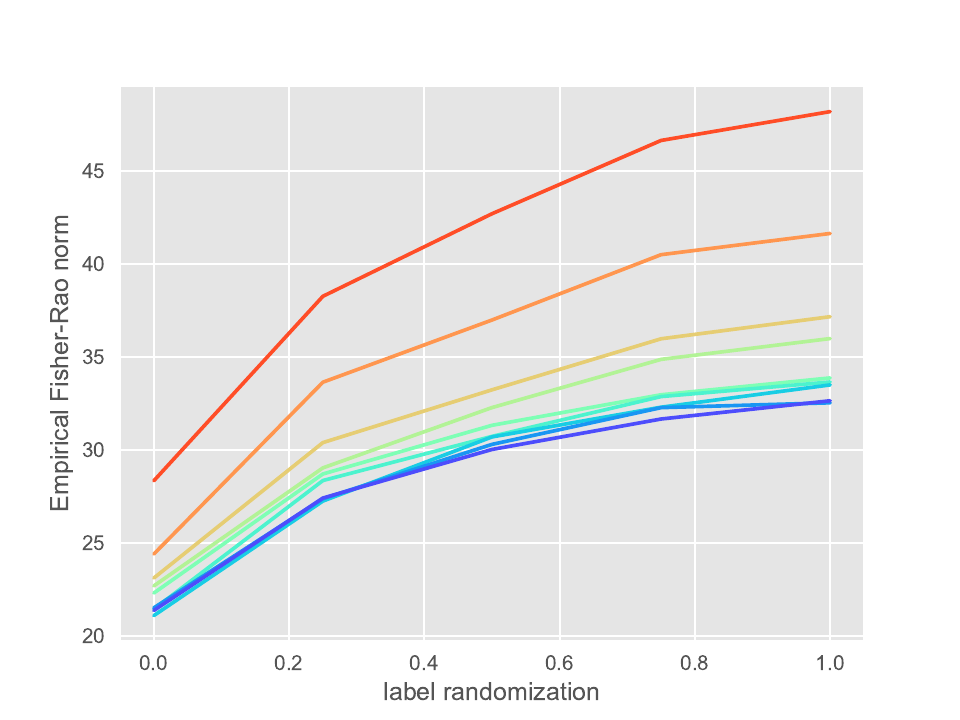}
\includegraphics[width=0.4\textwidth]{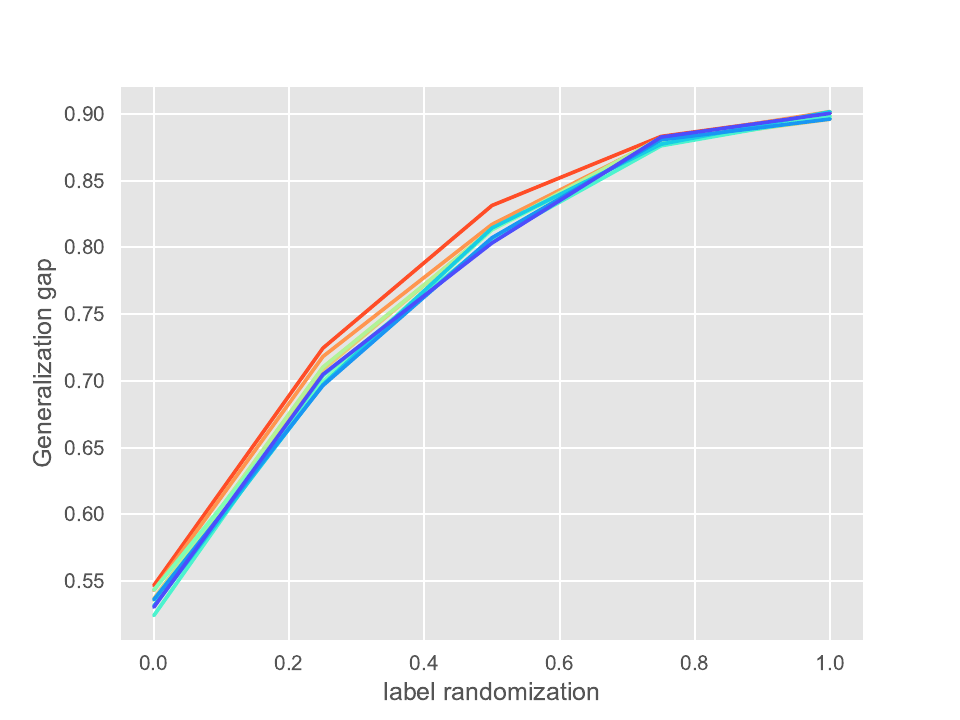}
\caption{Dependence of capacity measures on label randomization after optimizing with natural gradient descent. The colors show the effect of varying network width from $k=200$ (red) to $k=1000$ (blue) in increments of 100. The natural gradient optimization clearly distinguishes the network architectures according to their Fisher-Rao norm. \label{fig:kfac}}
\end{figure}

\begin{figure}[pht]
\centering
\includegraphics[width=0.3\textwidth]{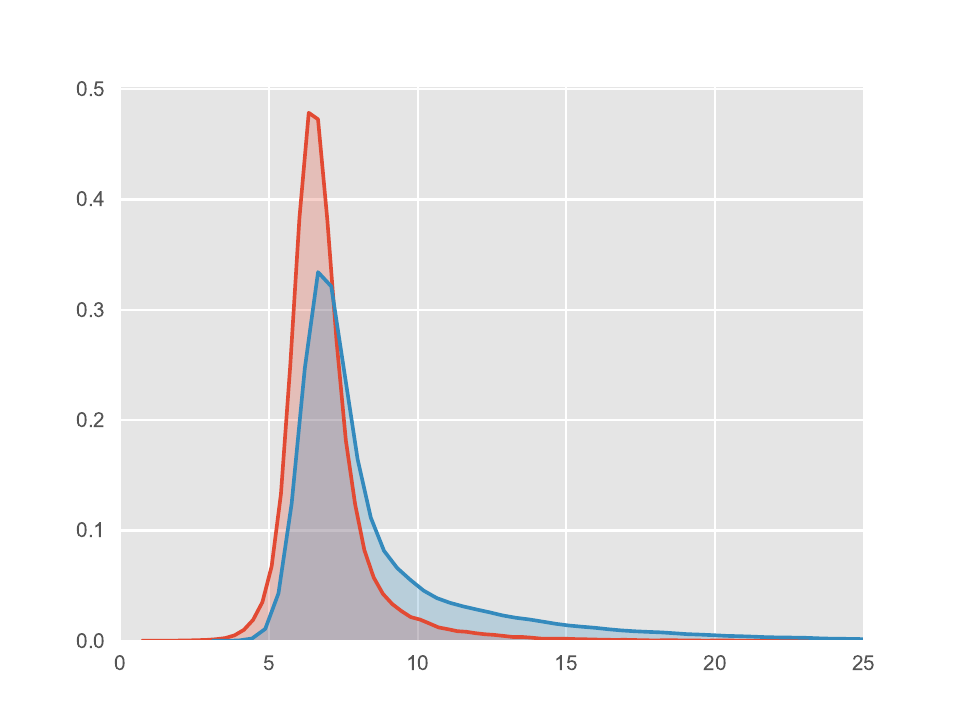}
\includegraphics[width=0.3\textwidth]{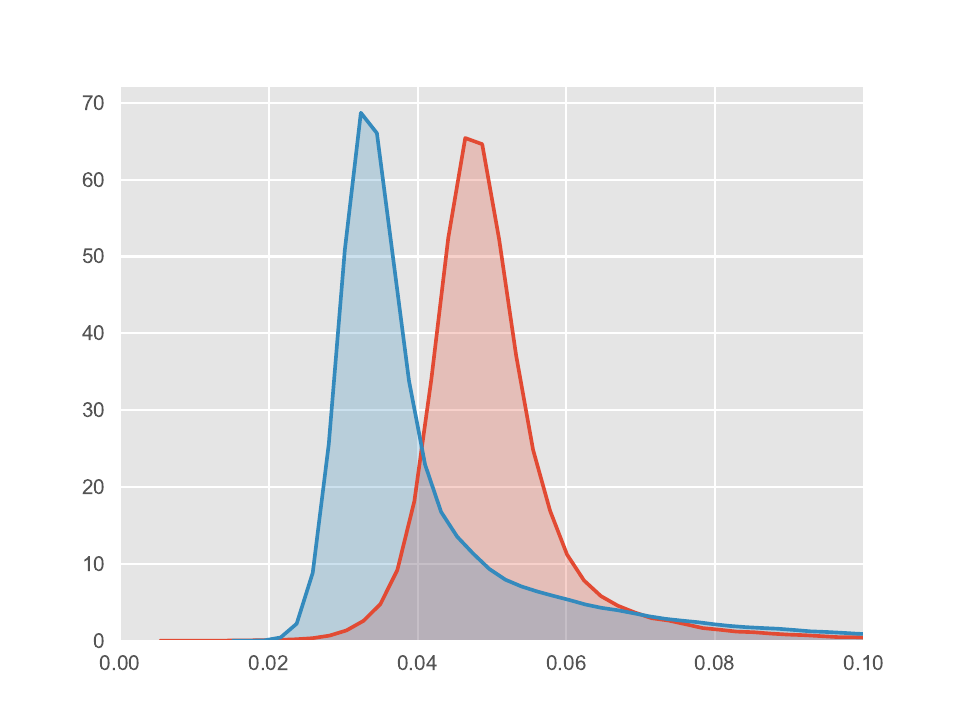}
\includegraphics[width=0.3\textwidth]{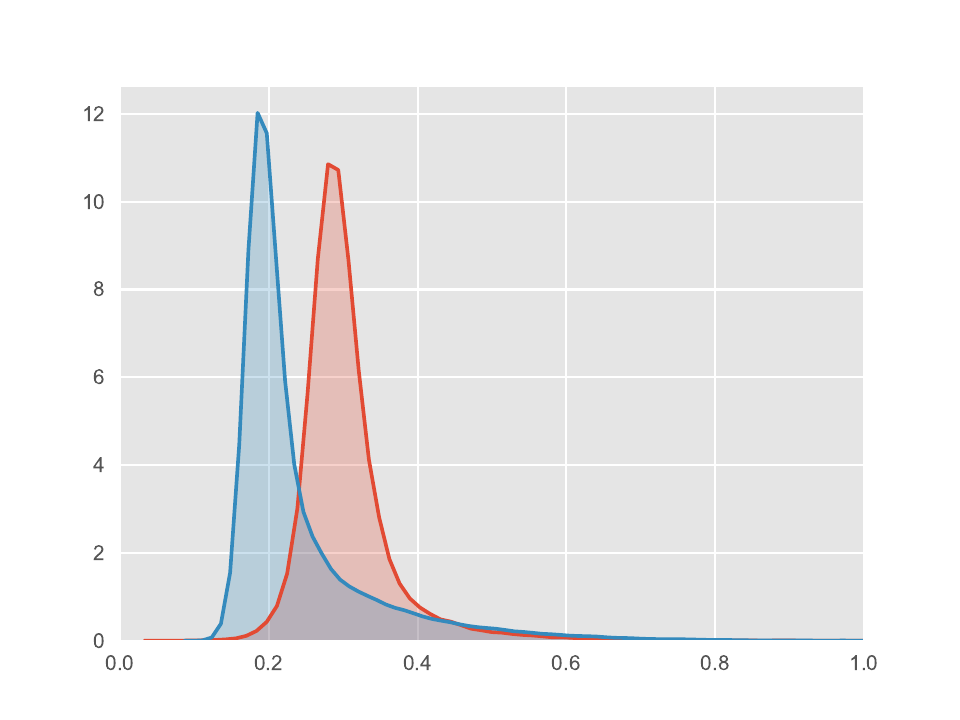}
\includegraphics[width=0.3\textwidth]{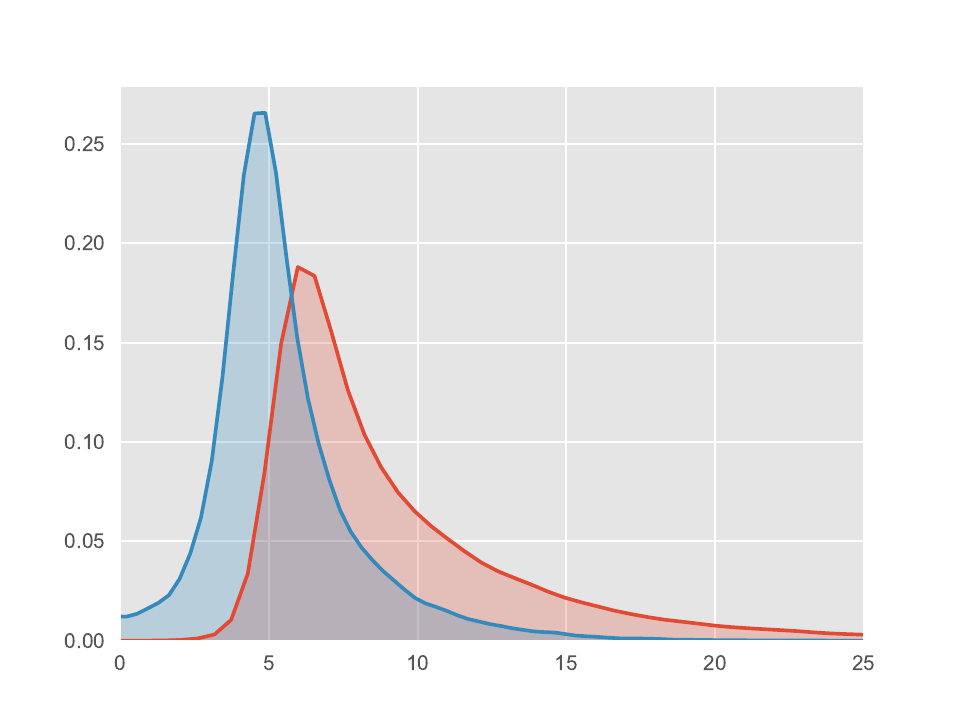}
\includegraphics[width=0.3\textwidth]{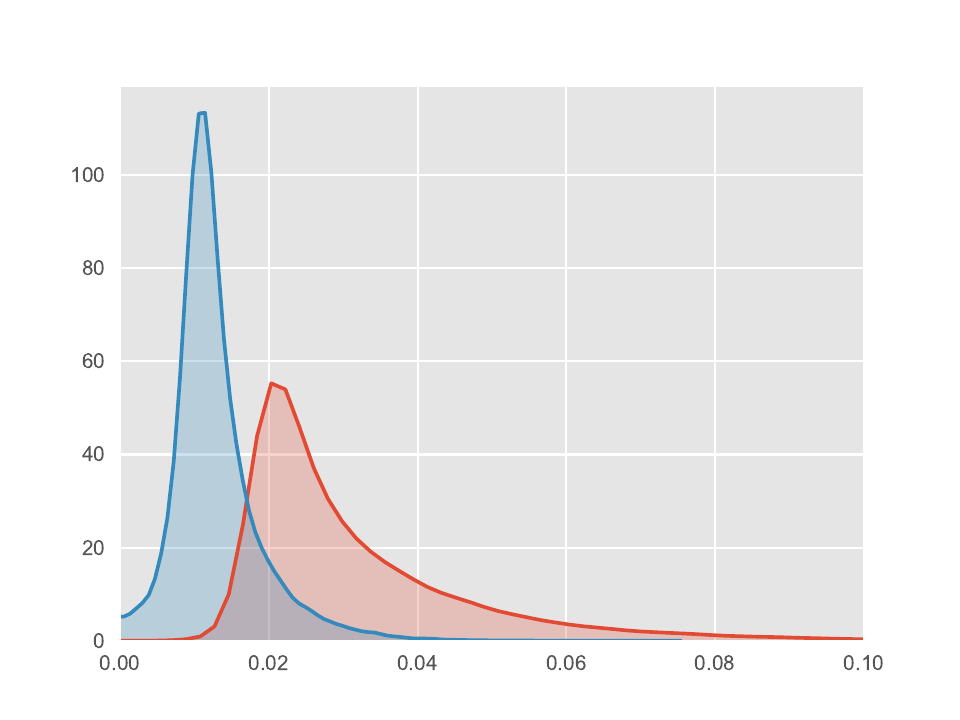}
\includegraphics[width=0.3\textwidth]{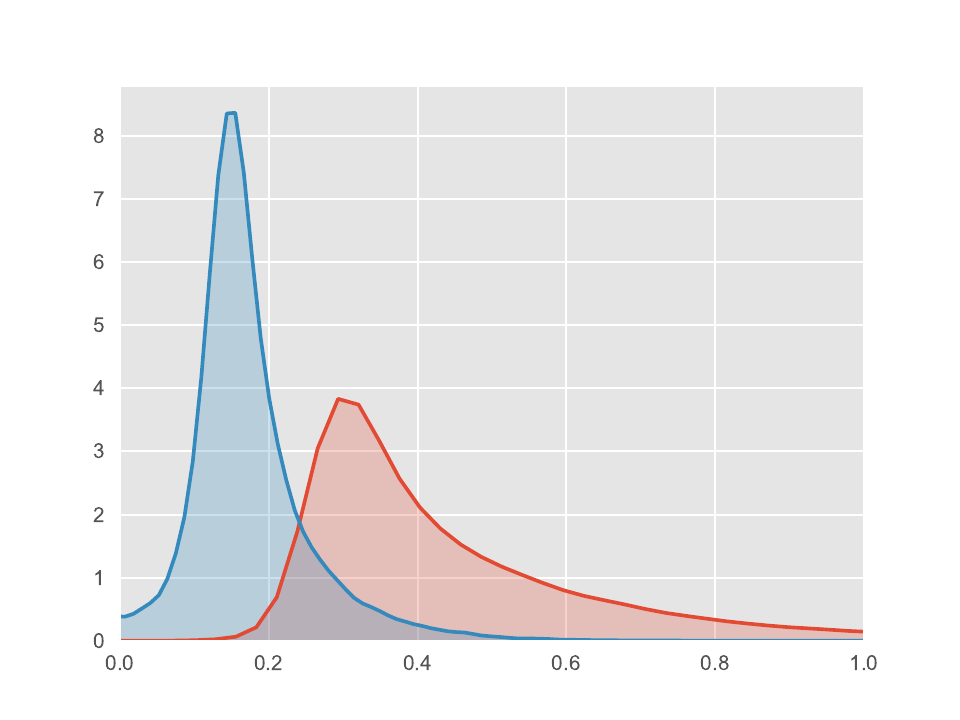}
\caption{Distribution of margins found by natural gradient (top) and vanilla gradient (bottom)  before rescaling (left) and after rescaling by spectral norm (center) and empirical Fisher-Rao norm (right).\label{fig:margin}}
\end{figure}

\begin{figure}[pht]
\centering
\includegraphics[width=0.23\textwidth]{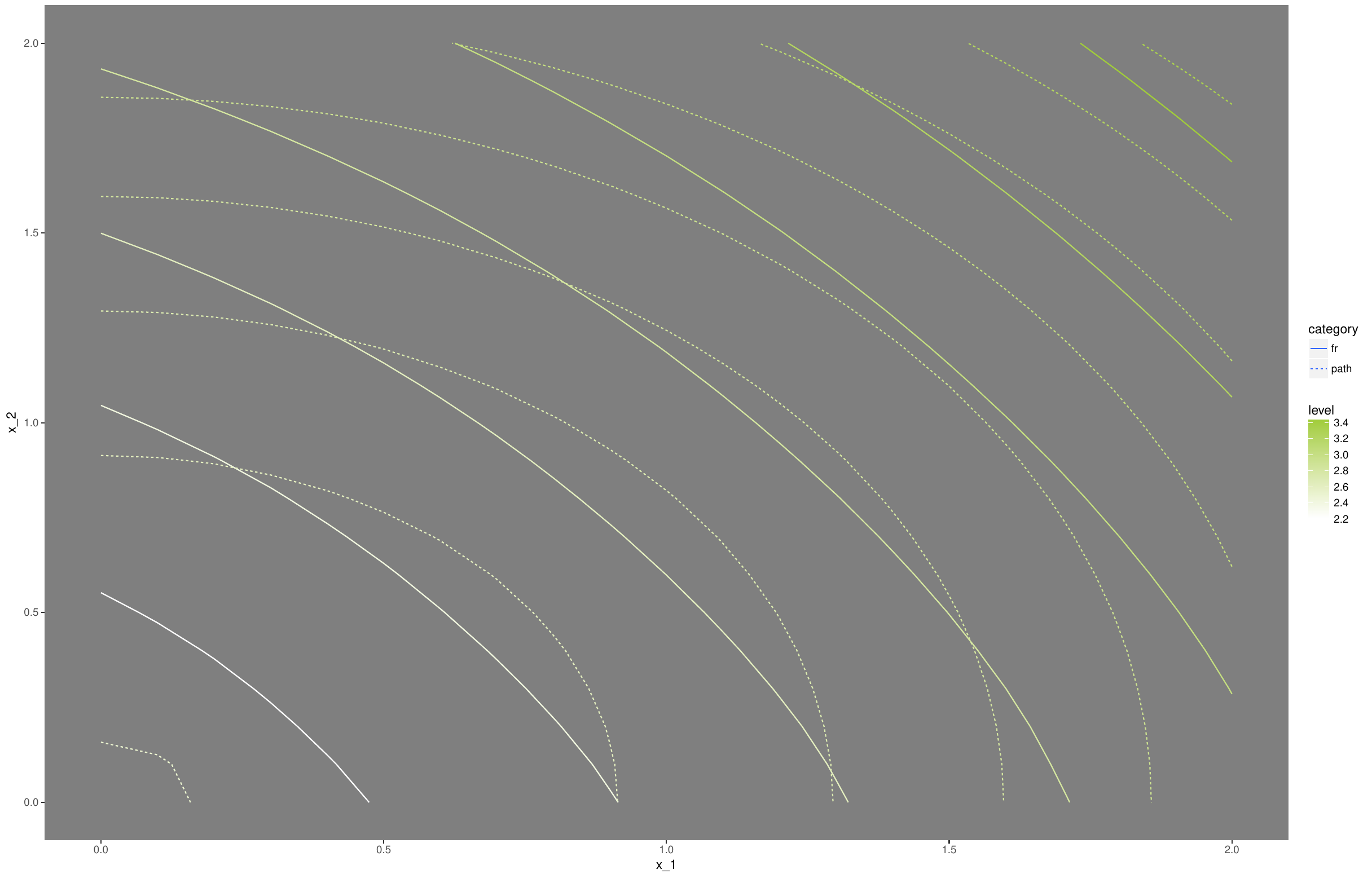}
\includegraphics[width=0.23\textwidth]{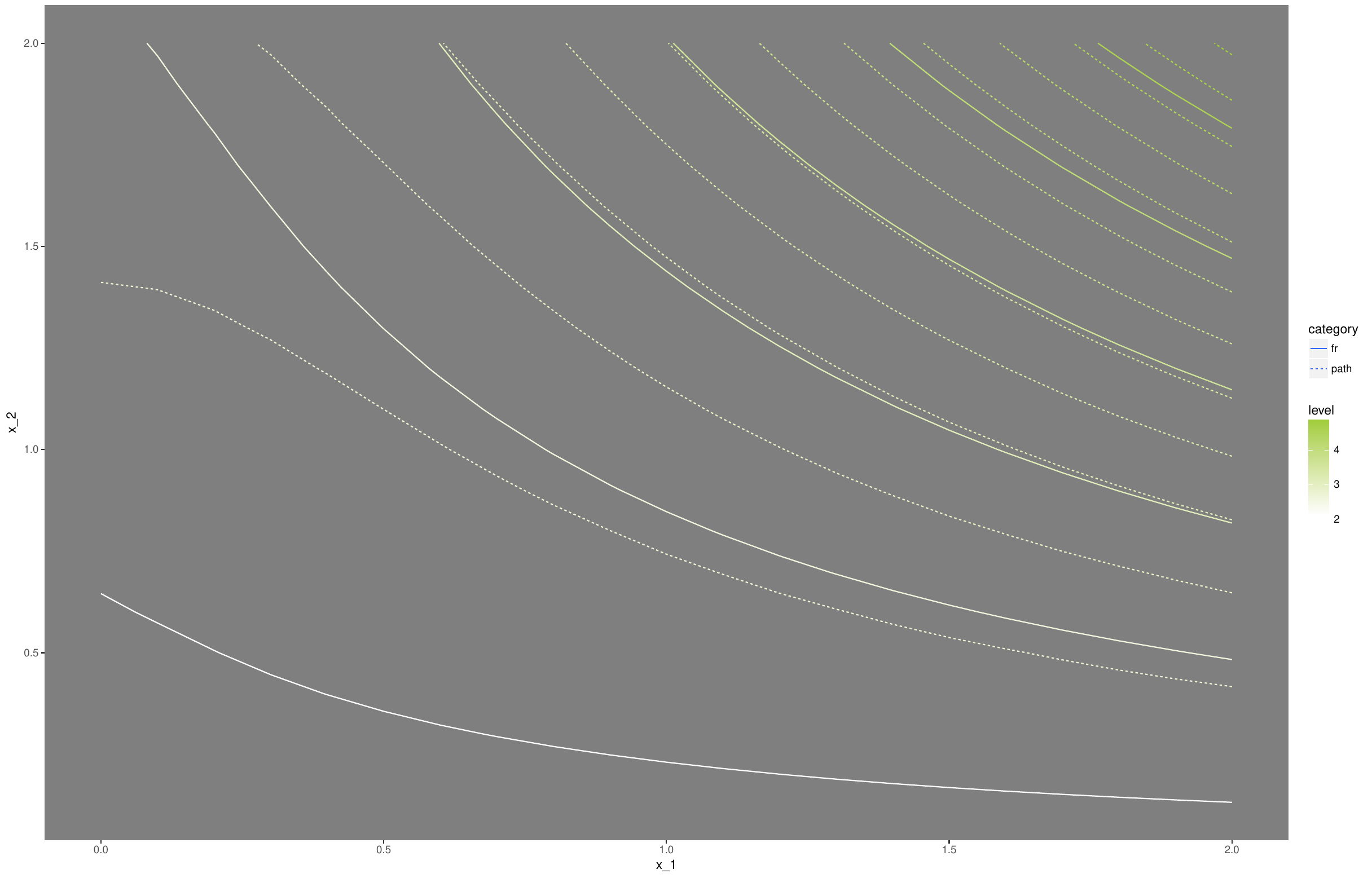}
\includegraphics[width=0.23\textwidth]{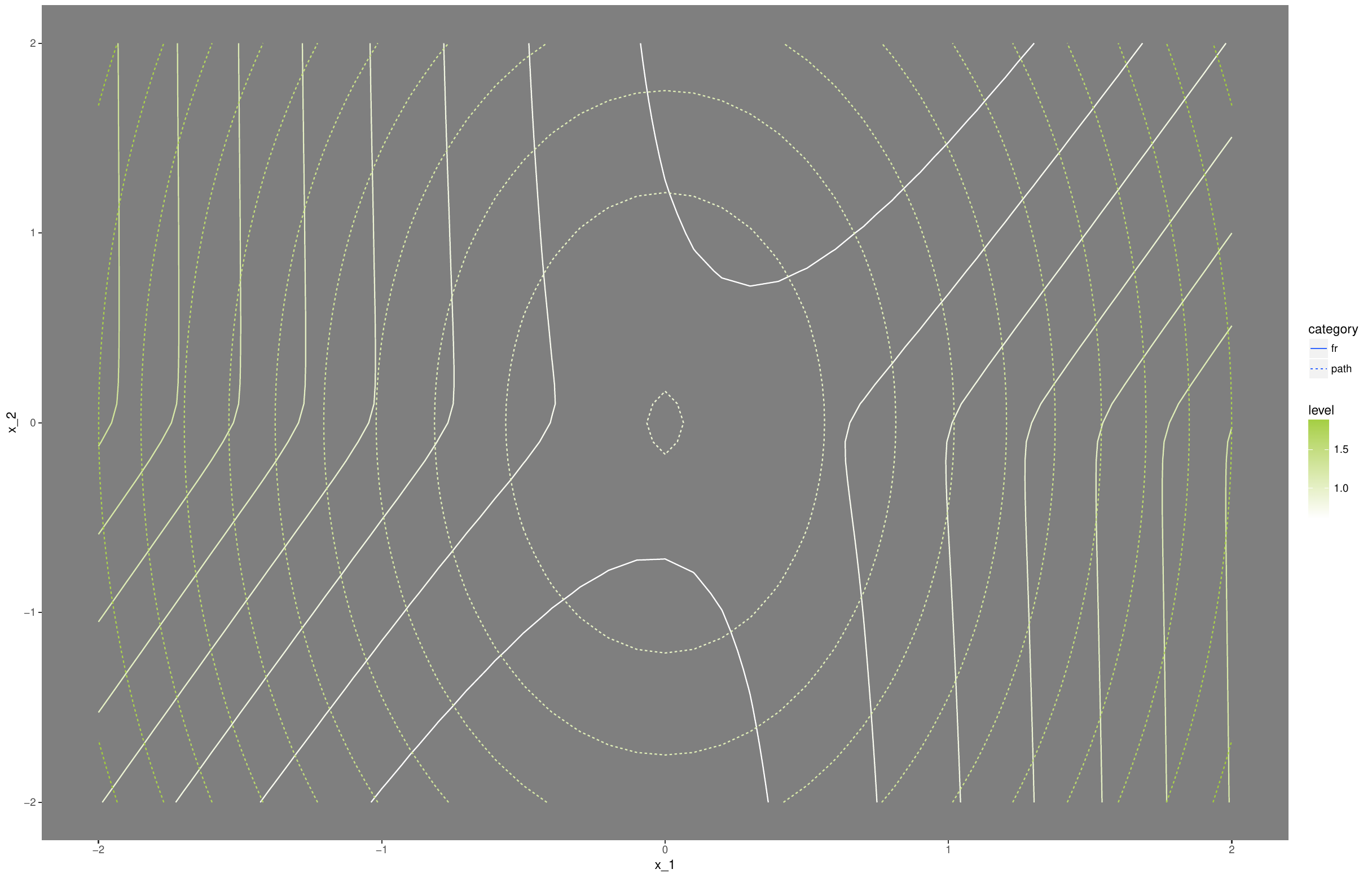}
\includegraphics[width=0.23\textwidth]{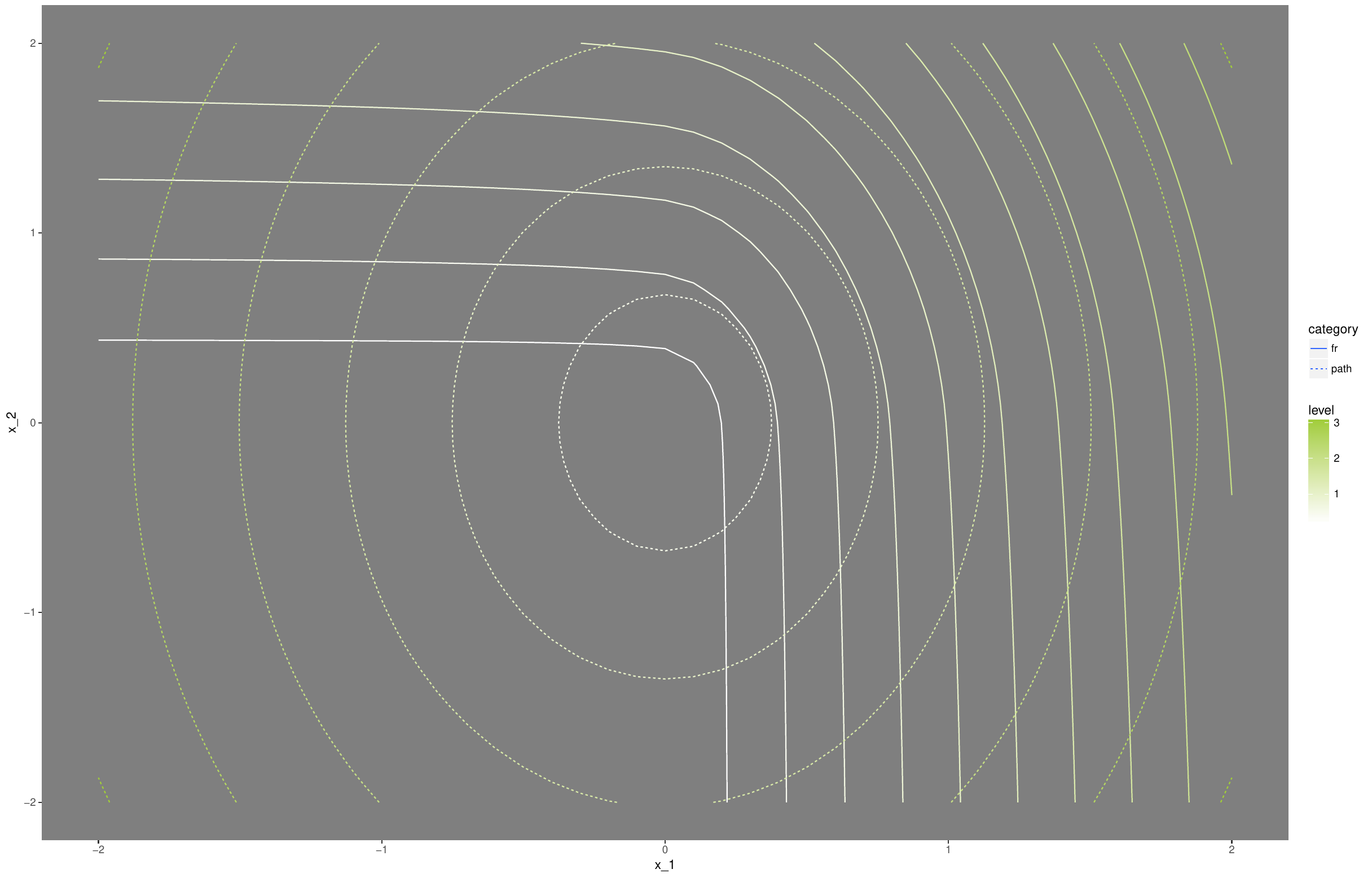}
\caption{The levelsets of Fisher-Rao norm (solid) and path-$2$ norm (dotted). The color denotes the value of the norm.}
\label{fig:contour}
\end{figure}

\begin{figure}[pht]
\centering
\includegraphics[width=60mm]{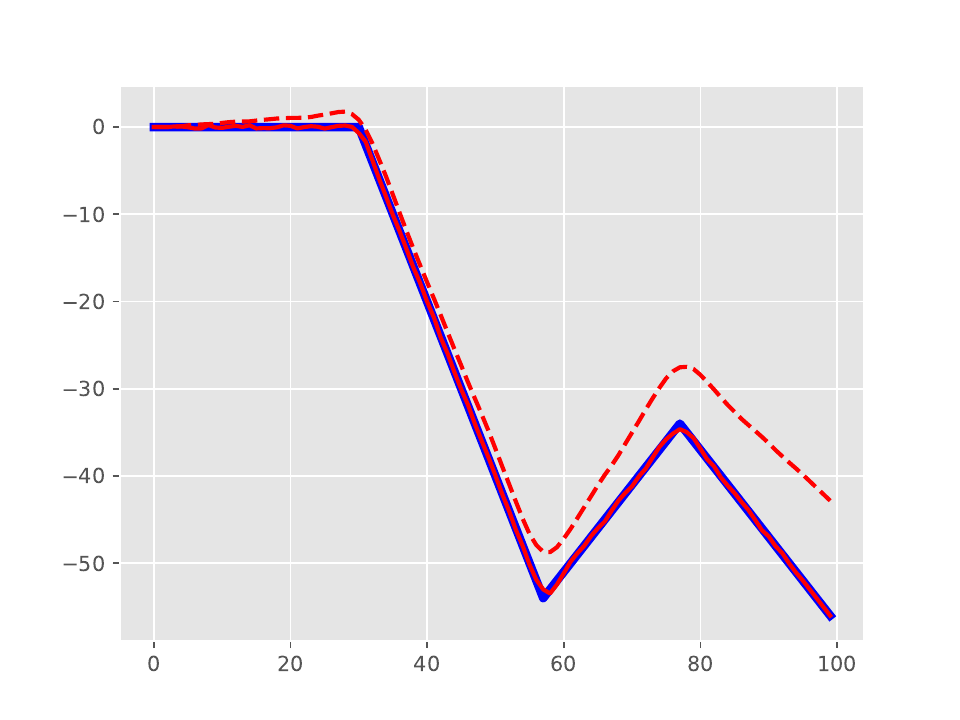}
\caption{Reproduction of conditioning experiment from \citep{shalev2017failures} after $10^4$ iterations of Adam (dashed) and K-FAC (red).}
\label{fig:failures}
\end{figure}

\end{document}